\documentclass[11pt]{article}

\usepackage[margin=1in]{geometry}
\usepackage{amsmath,amssymb,amsthm}
\usepackage{graphicx}
\usepackage{booktabs}
\usepackage{adjustbox}
\usepackage{caption}
\usepackage{xcolor}
\usepackage{comment}
\usepackage[round,authoryear]{natbib}
\usepackage{microtype}
\usepackage[hidelinks]{hyperref}

\newcommand{\set}[1]{\left\{#1\right\}}
\newcommand{\re}{\mathbb{R}}
\newcommand{\bN}{\mathbb{N}}
\newcommand{\bP}{\mathbb{P}}
\newcommand{\bE}{\mathbb{E}}
\newcommand{\bD}{\mathcal{D}}
\newcommand{\cA}{\mathcal{A}}
\newcommand{\cG}{\mathcal{G}}
\newcommand{\cH}{\mathcal{H}}
\newcommand{\cR}{\mathcal{R}}
\newcommand{\cX}{\mathcal{X}}
\newcommand{\cY}{\mathcal{Y}}
\renewcommand{\Tilde}[1]{\widetilde{#1}}

\newtheorem{definition}{Definition}[section]
\newtheorem{assumption}[definition]{Assumption}
\newtheorem{proposition}[definition]{Proposition}
\newtheorem{theorem}[definition]{Theorem}
\newtheorem{corollary}[definition]{Corollary}
\newtheorem{lemma}[definition]{Lemma}

\newenvironment{proofstep}[1]{%
  \par\medskip\noindent
  \begin{minipage}{\linewidth}
  \hrule\medskip\textbf{#1}\par\smallskip
}{%
  \medskip\hrule
  \end{minipage}\par\medskip
}

\title{Entropy Regularization: A Free Correction to Cross-Entropy for Verified Demonstrations}
\author{Mihir Dhanakshirur \quad Adam Ousherovitch \quad Ambuj Tewari\\
\small Department of Statistics, University of Michigan\\
\small \texttt{mihirgd@umich.edu} \quad \texttt{aoushero@umich.edu} \quad \texttt{tewaria@umich.edu}}
\date{}

\begin{document}

\maketitle

\begin{abstract}
Large language models are often post-trained on expert demonstrations using cross-entropy (CE), even when the downstream objective is not to imitate the demonstrated solution but to produce any output accepted by a verifier. This mismatch is seen in verifiable domains with multiple correct solutions, such as mathematical reasoning and code generation, where training data may contain only one expert solution per problem. We show that minimizing cross-entropy can be misaligned with minimizing verifier risk; two policies can assign identical likelihood to the observed demonstrations while placing different probability mass on incorrect outputs. This is formalized through a learning-theoretic counterexample in which CE minimization selects a suboptimal policy. We identify that controlling the support of the learned policy can solve this problem by preventing probability mass from spreading to unsupported outputs. Since support size is non-differentiable and computationally intractable, we propose \textbf{entropy-regularized cross-entropy (ER-CE)}, using token-level Shannon entropy as a tractable proxy. Finally, across mathematical reasoning and code-generation benchmarks, we find that entropy-regularized training consistently improves verifier accuracy over standard cross-entropy. Our results identify a simple failure mode of imitation-based post-training in verifiable tasks and provide a practical objective that is better aligned with producing correct outputs.
\end{abstract}
\section{Introduction}

Large Language Models (LLMs) learn from demonstrations in two regimes: pretraining and post-training. Pretraining is fundamentally about imitation; models are trained to mimic the next-token distribution of vast datasets, making cross-entropy (CE) loss a natural and effective objective \citep{radford2019language, brown2020language}. Post-training, however, often targets a fundamentally different goal: maximizing reward on specialized downstream tasks \citep{ouyang2022training,lambert2024tulu,deepseekai2025r1}. This contrast has real consequences in \textbf{verifiable} domains with \textbf{multiple correct solutions}, such as mathematical reasoning \citep{cobbe2021gsm8k} or code generation \citep{chen2021evaluating, austin2021program}. In these settings, training data typically provides a single expert demonstration per problem. However, the ultimate objective is not to copy the expert, but to generate \emph{any} solution that satisfies the verifier.

Despite these different goals, standard post-training with demonstrations frequently uses supervised fine-tuning (SFT) using CE \citep{chung2022scaling,touvron2023llama,wang2023selfinstruct}. While imitating an expert demonstrator is \emph{sufficient} to pass a verifier, cross-entropy is misaligned with the true goal of the task. By penalizing the model for any deviation from the provided demonstration, cross-entropy forces the LLM to model the idiosyncratic details of the specific expert, rather than learning to be correct. \citet{joshi2026learninganswercorrectdemonstrations} makes this precise by showing that CE fails to PAC-learn the verifier-reward when the policy class cannot perfectly model the demonstrator. However, they did not provide a tractable way to utilize this insight.

We provide a tractable method to correct for this misalignment: entropy regularization. Specifically, we \emph{minimize} entropy. The verifier-risk of a policy is exactly the probability mass it places \emph{outside} the set of correct answers. Cross-entropy pertains only to the answer that happened to be demonstrated. It says nothing about how the remaining mass is allocated, so two policies with identical likelihood on the training sample can assign arbitrarily different amounts of mass onto answers the verifier rejects. We give a precise learning-theoretic example of how this can lead to a suboptimal policy and show how favoring policies which minimize the entropy of the distribution naturally corrects for this. This motivates a tractable fine-tuning objective that costs no additional samples relative to standard SFT. Finally, we demonstrate its practical efficacy in improving verifier accuracy across benchmarks for math and coding.\\

\paragraph{Contributions.} Our key contributions are:
\begin{enumerate}
    \item We present a counterexample showing that minimizing CE on expert demonstrations need not align with the objective of learning a policy that satisfies the verifier (Proposition~\ref{prop: counter_example}).
    \item We introduce entropy-regularized cross-entropy (ER-CE) and establish conditions under which it improves upon standard CE minimization (Theorem~\ref{th: modify_ce}).
    \item Finally, across math problem-solving and code-generation tasks, we show empirically that entropy-regularized cross-entropy outperforms standard cross-entropy (Section~\ref{sec: exp}).
\end{enumerate}
\section{Related Work}
\subsection{Aligning Post-Training with the Downstream Objective}
\label{subsec: rw_alignment}

The observation that the loss minimized during training is a poor proxy for the quantity actually evaluated at deployment predates LLMs and has been explored for LLMs in other settings.

\paragraph{Decision-focused learning.} In the predict-then-optimize literature, a predictive model is fit upstream of an optimization problem whose solution quality is the real object of interest. \citet{elmachtoub2022smart} show that minimizing prediction error doesn't necessarily minimize downstream decision regret, and propose a surrogate loss with the decision problem in mind; \citet{donti2017task} and \citet{wilder2019melding} develop variants for different settings, and \citet{elbalghiti2019generalization} give formal guarantees about this framework. This literature parallels ours. Cross-entropy can be seen as a predictive loss, and verifier risk plays the role of the decision loss, with no guarantee that improving the former improves the latter.

\paragraph{Task-loss training for structured prediction.} A similar mismatch to ours was identified in neural sequence modeling, where we have access to a single output, but there are multiple acceptable outputs. \citet{och2003minimum} and \citet{shen2016minimum} optimize the evaluation metric directly, \citet{ranzato2016sequence} train at the sequence level rather than the token level, and \citet{norouzi2016reward} incorporate the task reward into maximum likelihood. This final method spreads mass \emph{outward} from the reference similar to how ours does.

\paragraph{Surrogate losses and consistency.} Proposition~\ref{prop: counter_example} can be read as a statement that log-loss is not a consistent surrogate for the verifier-risk over the constructed class. Thus, our work connects to the family of papers analyzing surrogate consistency \citep{bartlett2006convexity, tewari2007consistency}. We emphasize what is different here. In the standard classification setting, log-loss \emph{is} calibrated and target risk is defined using the conditional label distribution instead of the \emph{support} of the demonstrator.

\paragraph{Aligning fine-tuning with test-time strategy.} A recent line of work studies a related misalignment on the inference side, asking how training should change given a fixed decoding or search procedure. \citet{chow2024inference} fine-tune directly against a Best-of-$N$ verifier. \citet{chen2025rethinking} show that cross-entropy degrades pass@$N$ coverage and construct a loss aligned for pass@N. \citet{ousherovitch2026compute} generalize this, deriving aligned losses for arbitrary aggregation and filtering. Our misalignment is upstream of, and complementary to, this line. These methods derive new losses so that per-sample likelihood gains translate into gains after inference strategies are used. The gap we identify is present at pass@$1$ with no search. Both directions belong to the broader principle that one should train for the quantity the verifier actually
measures.

\subsection{Entropy Regularization of Output Distributions}
\label{subsec: rw_entropy}

Entropy penalties have a long history, though almost always with the opposite sign to ours. \citet{grandvalet2004semi} minimize the entropy of predictions on unlabeled data in a semi-supervised setting. Conversely, \citet{pereyra2017regularizing} \emph{penalize} low-entropy outputs as a regularizer for supervised learning; the same sign convention appears as the entropy bonus used to sustain exploration in reinforcement learning \cite{DBLP:journals/corr/abs-2103-06257, haarnoja2018soft}. Our result regularizes entropy in the opposite direction.

Most closely related, \citet{agarwal2025entropy} show that minimizing token-level entropy on unlabeled self-generated outputs substantially improves LLM performance on math, physics, and coding benchmarks. Their EM-FT objective and ours are similar in form. The contributions are, however, distinct in both mechanism and justification. Their method is label-free and operates on model outputs, while ours stems from supervised fine-tuning on labeled demonstrations; more importantly, their results are purely empirical, while we derive it, showing that the entropy penalty is a tractable surrogate for a support-minimizing loss that provably PAC-learns the verifier-risk (Theorem~\ref{th: modify_ce}) where cross-entropy provably does not (Proposition~\ref{prop: counter_example}).
\subsection{Learning from an Expert Policy}
\label{subsec: rw_expert}
The work of \citet{joshi2026learninganswercorrectdemonstrations} on learning from correct demonstrations is closely related to ours and considers the same underlying problem setting. Like us, they identify a mismatch between standard cross-entropy (CE) minimization and the objective of maximizing verifier success, and establish an impossibility result for CE-based learning. Their approach, however, differs fundamentally from ours. Rather than modeling the expert policy, they argue for directly learning the verifier and propose a version-space algorithm that PAC-learns it. We instead retain the policy-modeling perspective, motivated by the practical setting in which post-training begins from a pretrained model with rich learned representations rather than training a policy from scratch. Moreover, while their results provide a compelling theoretical case for departing from standard policy imitation, their proposed version-space method is generally computationally intractable. In contrast, our entropy-regularized objective can be incorporated into standard supervised fine-tuning with negligible additional overhead and yields substantial improvements in verifier accuracy.

\section{Limitation of Cross-Entropy}
\subsection{Problem Setup}

Let $\cX$ denote the space of questions/prompts and $\cY=\set{y_1,\dots,y_K}$ denote the finite space of answers/responses. Setting $\Delta(\cY)$ to be the set of distributions on $\cY$, we denote $\cH$ to be a class of functions from $\cX$ to $\Delta(\cY)$. For a hypothesis $h\in\cH$ and question $x\in\cX$, $h(\cdot|x)\in \Delta(\cY)$ denotes the distribution over answers and $h(y|x)$ is the probability assigned to answer $y\in\cY$ for question $x$. Finally, let $(\cX\times \cY)^*$ denote $\cup_{m\geq 1}(\cX\times \cY)^m$.

Fix a distribution $\bD$ on $\cX$ and hypothesis class $\cH$. We consider the setting where an unknown true hypothesis $h^*\in\cH$ is chosen to be the demonstrator and a sample $S:=\set{(X_i,Y_i)}$ is generated as follows:
\[X\sim \bD \text{ and }Y|X\sim h^*(\cdot|X).\]
The verifier is a function $V:\cX\times \cY\to \set{0,1}$ that indicates the correctness for every (question, answer) pair. For a question $x\in \cX$, denote the set of correct answers by
\[C_x:=\set{y\in \cY: V(x,y)=1}.\]
For a distribution $p\in\Delta(\cY)$, let $\text{supp}(p):=\set{y\in\cY:p(y)>0}$ denote the support of the distribution. We assume that our demonstrator $h^*$ satisfies 
\[\text{supp}(h^*(\cdot|x))\subset C_x \text{ for all }x\in\cX.\]
The above condition ensures that for every question, the demonstrator always reveals a correct answer. Given a sample $S$, we are tasked with finding a hypothesis in $\cH$ that minimizes the verifier-risk:
\[R(h):=\bE_{X\sim\bD, Y|X\sim h(\cdot|X)}[1_{\set{Y\notin C_x}}].\]
Our theoretical benchmark for evaluating learning algorithms is the notion of \emph{Probably Approximately Correct (PAC) learnability}.
\begin{definition}
    A hypothesis class $\cH$ is \textbf{PAC-learnable} with learning algorithm $\cA:(\cX\times \cY)^*\to \cH$ and sample complexity $m(\cdot,\cdot):(0,1)\times (0,1)\to \bN$, if for any $\epsilon,\delta\in (0,1)$ and $m\geq m(\epsilon,\delta)$, for any $\bD$ and $h^*\in\cH$,
    \[\bP_{S\sim (\bD\times h^*)^m}\left( R(\cA(S))>\inf_{h\in \cH}R(h)+\epsilon\right)< \delta.\]
\end{definition}
Informally, the PAC-learnability of an algorithm and a hypothesis class certifies that with high probability over the random training sample, the learned hypothesis has small verifier-risk. Practically, the most popular strategy is to select the hypothesis that minimizes the cross-entropy: for a sample $S=\set{(X_i,Y_i)}_{i=1}^m$, 
\[\cA^{\text{CE}}(S):=\arg\min_{h\in\cH}\frac{1}{m}\sum_{i=1}^m -\log h(Y_i|X_i).\]
\subsection{Counter-example}
We demonstrate that minimizing cross-entropy can be misaligned with minimizing verifier risk by constructing a hypothesis class for which \(\mathcal A^{\mathrm{CE}}\) fails to PAC-learn. The full proof is presented in Appendix~\ref{app: pf_counter}.
\begin{proposition}
Let $\cX=\bN$ and $\cY=\set{1,2,3}$. Additionally, let $\cG$ denote the set of functions from $\cX$ to $\set{1,2}$. Consider the following hypothesis class:
    \[\cH:=\set{u}\cup \set{h^g}_{g\in\cG},\]
    where
    \[u(\cdot|x)=\left(\frac{1}{2},\frac{1}{2},0\right),\, \forall\, x\in\cX\]
    and for each $g\in\cG$,
    \begin{align*}
        h^g(\cdot|x)=
        \begin{cases}
             \left(\frac{5}{8},\frac{1}{8},\frac{1}{4}\right) &\text{ if }g(x)=1\\
        \left(\frac{1}{8},\frac{5}{8},\frac{1}{4}\right) &\text{ if }g(x)=2.
        \end{cases}
    \end{align*}
    The empirical cross-entropy minimizer $\cA^{\text{CE}}$ fails to PAC-learn $\cH$.
    \label{prop: counter_example}
\end{proposition}
\begin{proof}[Proof sketch]
    
    Suppose $u$ is our true hypothesis and consider a sample $S=\set{(X_i,Y_i)}_{i=1}^m$ with all $X_i$ distinct. We can find a function $g\in\cG$ s.t. $g(X_i)=Y_i$ for every $i$ because $Y_i$ are drawn from a distribution $u(\cdot|X_i)$ whose support is $\set{1,2}$. The hypothesis $h^g\in\cH$ that corresponds to this $g$ will have a smaller cross-entropy loss on the sample than $u$ because:
    \[u(Y_i|X_i)=\frac{1}{2}<\frac{5}{8}=h^g(Y_i|X_i)\]
    Despite $u$ being the true hypothesis, the  empirical minimizer of the cross-entropy loss will always be $h^g$. As a result, $\cA^{\text{CE}}$ incurs a non-zero risk because it assigns a weight of $\frac{1}{4}$ outside the support of $u$. The proof concludes by showing that for every $m$, there exists a distribution on $\cX$ that, with a probability bounded away from zero, draws distinct $X_1,\dots,X_m$ for every $m$.
\end{proof}
There are a couple of points worth noting here. First, our proof crucially relies on the asymmetry between cross-entropy and the verifier risk; cross-entropy depends only on the weight assigned to a sampled $Y_i$ while computing the verifier risk involves probabilities on the set $\cY\setminus\set{Y_i}$. Secondly, our hypothesis class $\cH$ is rich when viewed as a class of probabilistic hypotheses but is in fact quite simple when we ignore the probabilities and focus only on the supports. If we define
\[\cR:=\set{x\mapsto \text{supp}(h_x):h\in \cH},\]
then in Proposition~\ref{prop: counter_example}, 
\[\cR=\set{x\mapsto \set{1,2}} \cup \set{x\mapsto \set{1,2,3}},\]
so $|\cR|=2$ even though $|\cH|$ grows exponentially in $m$. Interestingly, the hypothesis class $\cH$ above can be PAC-learned by $\cA^{\text{trivial}}$, the algorithm that ignores probabilities and always returns hypothesis $u$. This follows from two simple observations: (1)
$\text{supp}(u(\cdot|x))\subset \text{supp}(h(\cdot|x)) \text{ for all }x\in\cX,h\in\cH$ and (2) if the support of our selected hypothesis is contained in the support of the true hypothesis for every question, then its verifier-risk is $0$. Even though $\cA^{\text{trivial}}$ cannot be generalized beyond the counter-example in Proposition~\ref{prop: counter_example}, it motivates a loss function that explicitly accounts for the support of the candidate hypothesis.
\section{Entropy Regularized Cross-Entropy}
\label{sec: proof}
\begin{definition}
    For a hypothesis class $\cH\subset (\Delta(\cY))^\cX$, the \textbf{induced support class} is given by
    \[\cR_\cH:=\set{x\mapsto \text{supp}(h_x):h\in \cH}.\]
\end{definition}
$\cH$ in Proposition~\ref{prop: counter_example} was an example of a hypothesis class with an induced support class of size $2$. We generalize the setting in the counter-example by studying the PAC-learnability of hypothesis classes whose induced support class is finite.
\begin{assumption}
    For a hypothesis class $\cH$, the induced support class is finite, that is, $|\cR_\cH|<\infty$.
    \label{ass: finite_support}
\end{assumption}
Under a minimal assumption, we prove that such hypotheses classes can be PAC-learned by modifying the cross-entropy loss.\footnote{This assumption is true in the LLM-verifier setting with a finite vocabulary and maximum sequence length.}
\begin{assumption}
\label{ass: lower_bound}
    For a hypothesis class $\cH$, $\exists\,\eta>0$ s.t. $\forall\, x\in \cX, y\in \cY, h\in \cH$,
    \[h_x(y)\in \set{0}\cup [\eta,1],\]
\end{assumption}
The counter-example in Proposition~\ref{prop: counter_example} satisfies Assumption~\ref{ass: lower_bound} with $\eta=\frac{1}{8}$. This assumption is also natural in the LLM-Verifier setting because the soft-max operation prevents a model from assigning zero probability to a sequence of tokens.
\begin{definition}
    For $\alpha\neq 0,1$, define
    \[S_\alpha(p):=(\sum_{y\in\cY}p^\alpha(y))^{\frac{1}{1-\alpha}}\, \text{ for }p\in \Delta(\cY).\]
    Since $\lim_{\alpha\to 0}S_\alpha(p)=|\text{supp}(p)|$ and $\lim_{\alpha\to 1}S_\alpha(p)=e^{-\sum_{y\in\cY}p(y)\log(p(y))}$, we define
    \[S_0(p):=|\text{supp}(p)| \text{ and }S_1(p):=e^{-\sum_{y\in\cY}p(y)\log(p(y))} \text{ for }p\in \Delta(\cY).\]
\end{definition}
For a distribution $p$, $(S_\alpha(p))_{\alpha> 0}$ is a family of surrogates for the support of $p$ that are also differentiable with respect to $p$. This family is closely connected to entropy-like measures: $S_1(p)=e^{H(p)}$ where $H(\cdot)$ is the Shannon Entropy and for $\alpha>0,\alpha\neq 1$, $S_\alpha(p)=e^{R_\alpha(p)}$ where $R_\alpha(\cdot)$ is the Renyi Entropy of order $\alpha$. For a further discussion on $(S_\alpha(\cdot))$ as a proxy for the support, see Appendix~\ref{app: discussion_s_alpha}.\\

\begin{theorem}
\label{th: modify_ce}
    Suppose $\cH$ satisfies Assumption~\ref{ass: finite_support} and Assumption~\ref{ass: lower_bound}.
    Consider the loss $l_{\lambda,\alpha}:\Delta(\cY)\times \cY\to \bar{\re}_+=[0,\infty]$,
    \[l_{\lambda,\alpha}(p,y) = -\log(p(y))+\lambda S_\alpha(p).\]
    Then the algorithm $\cA$ that, for a given sample $S=\set{(X_i,Y_i)}_{i=1}^m$, returns
    \[\cA(S)\in \arg\min_{h\in\cH}\frac{1}{m}\sum_{i=1}^m l_{\lambda_m,\alpha_m}(h(\cdot|X_i),Y_i),\]
    for sequences $(\lambda_m)$ and $(\alpha_m)$ s.t. $\lambda_m\uparrow\infty$ and $0<\alpha_m<1,\, \alpha_m\downarrow 0$, PAC-learns $\cH$.
\end{theorem}
The full proof is presented in Appendix~\ref{app: pf_modify_ce}.
\begin{proof}[Proof Sketch]
The proof proceeds through three steps, progressively replacing $1_{\set{y\notin \text{supp}(p)}}$ with $-\log(p(y))$ and $|\text{supp}(p)|$ with $S_\alpha(p)$.
\begin{proofstep}{Step 1: Indicator function + Support Penalty}
    Show that $\cH$ is PAC-learnable with the loss:
\[l(p,y) = \frac{1}{\eta}1_{\set{y\notin \text{supp}(p)}}+|\text{supp}(p)|.\]
\end{proofstep}
\begin{proofstep}{Step 2: Cross-Entropy + Support Penalty}
    Show that $\cH$ is PAC-learnable with the loss:
    \[l_\lambda(p,y) = -\log(p(y))+\lambda|\text{supp}(p)|,\]
    if $(\lambda_m)$ is chosen s.t. $\lambda_m\uparrow \infty$.
\end{proofstep}
\begin{proofstep}{Step 3: Cross-Entropy + Entropy Penalty}
    Show that $\cH$ is PAC-learnable with the loss:
    \[l_{\lambda,\alpha}(p,y) = -\log(p(y))+\lambda|S_\alpha(p)|,\]
    if $(\lambda_m),(\alpha_m)$ are chosen s.t. $\lambda_m\uparrow \infty$ and $\alpha_m\downarrow 0$.
\end{proofstep}
    We present the proof idea for \textbf{Step 1}, as it illustrates why adding a support (or support-like) penalty to losses that depend only on the observed label, such as $1_{\set{y\in \text{supp}(p)}}$ or $-\log(p(y))$), yields learnability guarantees that the un-regularized losses otherwise lack.\\
    The proof relies on two key ideas:
    \begin{enumerate}
        \item[\textbf{(1)}] Denote the expected risk of the loss $l$ by $\Tilde{R}$. Minimizing the difference $\Tilde{R}(h)-\Tilde{R}(h^*)$ is sufficient for minimizing the verifier risk.
        \item[\textbf{(2)}] The loss $l(p,y)$, which depends only on the underlying supports of the hypotheses, will generalize as the sample size increases.
    \end{enumerate}
    Throughout we will use the notation $r_h(x):=\text{supp}(h(\cdot|x))$. For every $x\in\cX$,
    \begin{align*}
        \sum_{y\notin r_{h^*}(x)}h(y|x) \leq \left|r_h(x)\setminus r_{h^*}(x)\right|= \left|r_{h^*}(x)\setminus r_h(x)\right|+\left|r_h(x)\right|-\left|r_{h^*}(x)\right|.
    \end{align*}
    and
    \begin{align*}
    \left|r_{h^*}(x)\setminus r_h(x)\right| = \sum_{y\notin r_h(x)}1_{\set{y\in r_{h^*}(x)}}
        \leq \sum_{y\notin r_h(x)}\frac{h^*(y|x)}{\eta}
        = \frac{1}{\eta}\bP_{Y\sim h^*(\cdot|x)}\left(Y\notin r_h(x)\right).
    \end{align*}
    where the inequality follows from Assumption~\ref{ass: lower_bound}. Together, the two inequalities allows us to bound the verifier-risk by a surrogate risk for any hypothesis $h\in\cH$:
    \[R(h)-\underbrace{R(h^*)}_{=0} \leq \Tilde{R}(h)-\Tilde{R}(h^*),\] 
    where
    \[\Tilde{R}(h):=\bE_{\bD,h^*}\left[\frac{1}{\eta}1_{\set{Y\notin r_h(X)}}\right]+\bE_{\bD}[\left|r_h(X)\right|.\]
    This shows \textbf{(1)}. It remains to show \textbf{(2)}: that the minimizer of the empirical loss $l$, PAC-learns $\cH$ under the surrogate risk. The loss function class $L_\cH:=\set{(x,y)\mapsto l(h(\cdot|x),y):h\in \cH}$ is finite by Assumption~\ref{ass: finite_support} and the loss function $l$ can be uniformly bounded for all $x\in\cX,y\in\cY$ and $h\in\cH$ by $\frac{2}{\eta}$.  A standard application of Hoeffding's Inequality gives PAC-learnability with sample complexity
    \[m(\epsilon,\delta)=\left\lceil\frac{8}{\eta^2\epsilon^2}\log\left(\frac{2|\cR|}{\delta}\right)\right\rceil.\]
    This transformation is crucial because even though $\cH$ may be infinite and cross-entropy may overfit to the sample, the class of functions $L_\cH$ is finite and will generalize as the sample size increases. 
\end{proof}
The following result immediately follows from Theorem~\ref{th: modify_ce}.
\begin{corollary}
    If we set $\lambda_m=m$ and $\alpha=1/m$, then the algorithm $\cA$ that, for a given sample $S=\set{(X_i,Y_i)}_{i=1}^m$, returns
    \[\cA(S)\in \arg\min_{h\in\cH}\frac{1}{m}\sum_{i=1}^m l_{\lambda_m,\alpha_m}(h(\cdot|X_i),Y_i),\]
    PAC-learns $\cH$, with sample complexity
    \[m(\epsilon,\delta)=\left\lceil \max\left\{\frac{32}{\eta^2\epsilon^2}\log\left(\frac{3|\cR|}{\delta}\right),\frac{4}{\epsilon}\log\left(\frac{1}{\eta}\right),1+\frac{\log\eta}{\log\left(1-\frac{\eta\epsilon}{4}\right)}\right\} \right\rceil.\]
\end{corollary}
\section{Experiments}
\label{sec: exp}
We validate our theory by fine-tuning an LLM with entropy-regulated loss and evaluating it in two settings: math problems (Section~\ref{subsec: math_problems} and Section~\ref{sec:math7b}) and code generation (Section~\ref{subsec: code_generation}). These two settings closely align with our theoretical setup. At training, for a given problem, we are provided with only one out of many possible solutions.
Concretely, given a question-answer pair $(x,\mathbf{y})$, where $\mathbf{y}=(y_1,\dots,y_T)$, and model $\pi$, the loss is given by:
\[\frac{1}{T}\sum_{i=1}^T -\log\left(\pi(y_i|x,y_{1:i-1})\right) +\frac{\lambda}{T}H\left(\pi(\cdot|x,y_{1:i-1})\right),\]
where we use the convention that $y_{1:i}=(y_1,\dots,y_i)$ with $y_{1:0}=\phi$.\\

For solving math problems, this means that we're shown one possible approach to obtain the final numerical answer. In code generation, we're shown one correct Python program for a given text description of a program. At test time, the model is evaluated on whether it produces a solution that passes the verifier, not whether it matches the exact solution in the dataset. Specifically, for the math problems task, the model's success depends on the final numerical solution being correct and is agnostic to the reasoning trace. In code generation, the synthesized program is only required to pass the test cases and be functionally correct. We include experimental details in Appendices~\ref{sec:appendix_setup}, \ref{sec:appendix_mbpp_setup}, and \ref{app:math} and relevant ablations in Appendices~\ref{app:sampled}, \ref{app:renyi}, and \ref{app:answeronly}. These provide additional support for our proposed mechanism. When CoT reasoning traces are replaced with answer-only targets, removing multiple paths to the correct answers, gains largely disappear (Appendix~\ref{app:answeronly}). Under temperature decoding, which is closer to our theoretical setup, gains are even larger than under the greedy decoder we use in this section (Appendix~\ref{app:sampled}). Finally, improvement is robust across a range of Renyi order parameters $\alpha \in [0.5, 2]$ (Appendix~\ref{app:renyi}).

\subsection{Math Problems}
We evaluate the utility our loss on the GSM8K dataset \citep{cobbe2021gsm8k}, which contains high-school level math problems where the correct answer is always a number. The training dataset which consists of (problem, Chain-of-Thought answer) pairs is used to fine-tune the model and the test set is used to verify whether the model has given a correct to a problem by checking if the solution boxes the correct numerical answer.

We fine-tune a Qwen2.5-1.5B-Instruct model \citep{yang2024qwen2} with our entropy-regularized loss using Low-Rank Adaptation (LoRA) \citep{hu2022lora} across a grid of penalty strengths. Performance is evaluated using the Pass@1 coverage metric. 
\begin{figure}[h]
    \centering

    \begin{adjustbox}{valign=t}
    \begin{minipage}{0.52\textwidth}
        \centering
        \includegraphics[width=\linewidth]{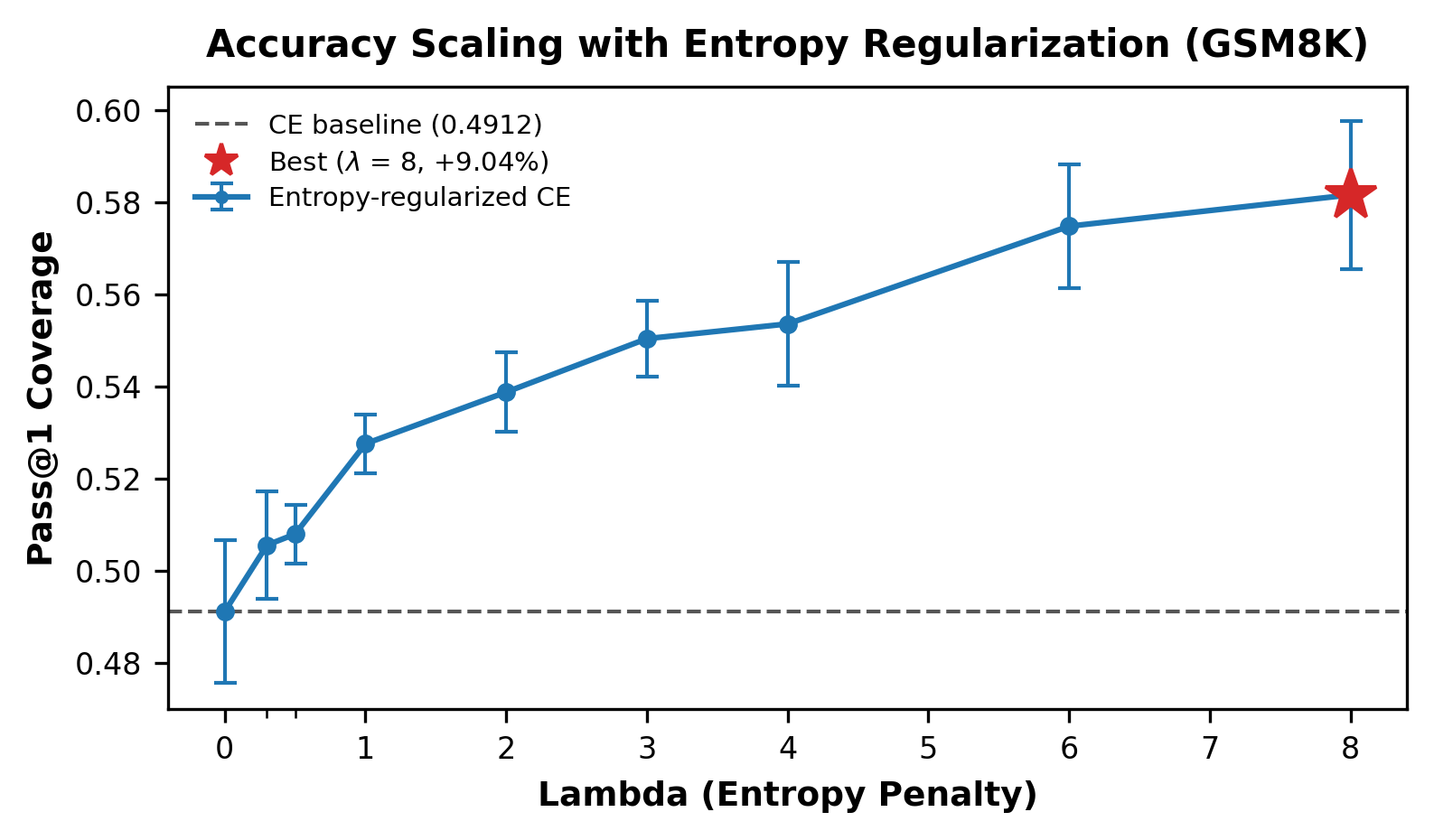}
        \caption{\textbf{Accuracy monotonically increases with penalty strength.} Pass@1 accuracy on GSM8K (CoT) versus entropy penalty $\lambda$. Error bars show $\pm 1$ standard deviation over 5 seeds.}
        \label{fig:accuracy_scaling}
    \end{minipage}
    \end{adjustbox}
    \hfill
    \begin{adjustbox}{valign=t}
    \begin{minipage}{0.43\textwidth}
        \centering
        \small

        \begin{tabular}{lcc}
            \toprule
            $\lambda$ & \textbf{Mean} & \textbf{Std. Dev.} \\
            \midrule
            $0.0$ (Baseline) & 0.4912 & $\pm 0.0154$ \\
            $0.3$            & 0.5056 & $\pm 0.0116$ \\
            $0.5$            & 0.5080 & $\pm 0.0064$ \\
            $1.0$            & 0.5276 & $\pm 0.0064$ \\
            $2.0$            & 0.5388 & $\pm 0.0086$ \\
            $3.0$            & 0.5504 & $\pm 0.0083$ \\
            $4.0$            & 0.5536 & $\pm 0.0134$ \\
            $6.0$            & 0.5748 & $\pm 0.0134$ \\
            $\mathbf{8.0}$   & \textbf{0.5816} & $\mathbf{\pm 0.0161}$ \\
            \bottomrule
        \end{tabular}

        \captionof{table}{\textbf{Accuracy peaks $\approx9\%$ over standard cross entropy.} GSM8K Pass@1 accuracy under varying entropy penalties,
        aggregated over 5 seeds.}
        \label{tab:entropy_main_results}
    \end{minipage}
    \end{adjustbox}
\end{figure}
\label{subsec: math_problems}

\paragraph{Results.} As seen in Figure~\ref{fig:accuracy_scaling} and Table~\ref{tab:entropy_main_results}, performance monotonically increases with the strength of the penalty among the $\lambda$ we tested. At its peak, ER-CE outperforms standard CE by around $9\%$ at no added training cost.
\subsection{Code Generation}
We further evaluate the empirical performance of our loss on the Mostly Basic Python Problems (MBPP) dataset \citep{austin2021program}. This dataset contains introductory-level Python programming tasks where correctness is verified dynamically by executing the generated code against a suite of hidden `assert` statements. 

We fine-tune a Qwen2.5-1.5B-Instruct model \citep{yang2024qwen2} with our entropy-regularized loss using Low-Rank Adaptation (LoRA) \citep{hu2022lora} across a grid of penalty strengths ($\lambda \in [0.0, 6.0]$). Performance is evaluated using Greedy Pass@1 accuracy over a sandboxed execution environment.

\begin{figure}[h]
    \centering

    \begin{adjustbox}{valign=t}
    \begin{minipage}{0.52\textwidth}
        \centering
        \includegraphics[width=\linewidth]{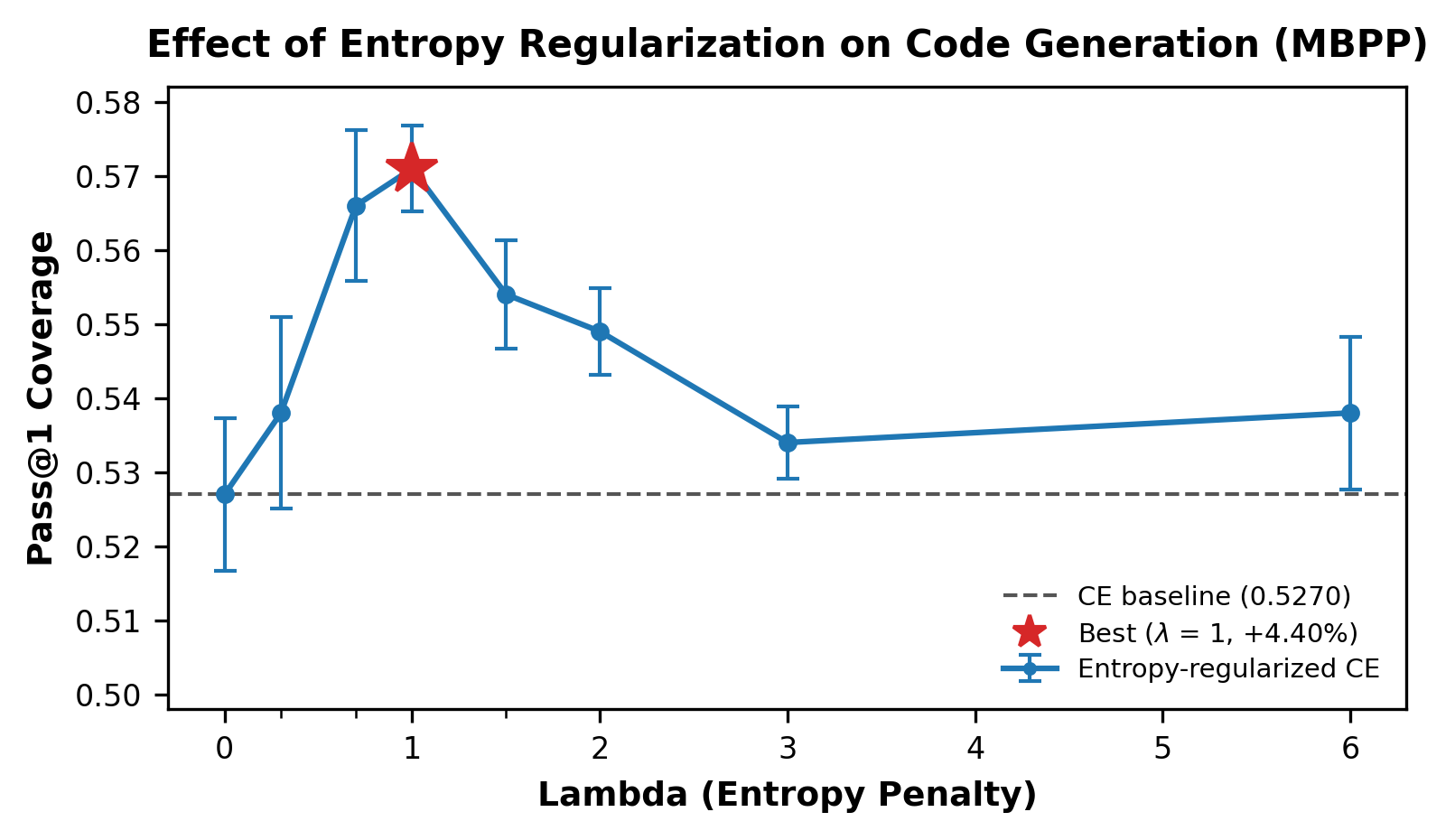}
        \caption{\textbf{ER-CE uniformly beat out CE across all $\lambda$.} Greedy Pass@1 accuracy on MBPP versus entropy penalty $\lambda$. Error bars show $\pm 1$ standard deviation over 5 seeds.}
        \label{fig:mbpp_entropy_curve}
    \end{minipage}
    \end{adjustbox}
    \hfill
    \begin{adjustbox}{valign=t}
    \begin{minipage}{0.43\textwidth}
        \centering
        \small

        \begin{tabular}{lcc}
            \toprule
            $\lambda$ & \textbf{Mean} & \textbf{Std. Dev.} \\
            \midrule
            $0.0$ (Baseline) & 0.5270 & $\pm 0.0103$ \\
            $0.3$            & 0.5380 & $\pm 0.0129$ \\
            $0.7$            & 0.5660 & $\pm 0.0102$ \\
            $\mathbf{1.0}$            & \textbf{0.5710} & $\mathbf{\pm 0.0058}$ \\
            $1.5$            & 0.5540 & $\pm 0.0073$ \\
            $2.0$            & 0.5490 & $\pm 0.0058$ \\
            $3.0$            & 0.5340 & $\pm 0.0049$ \\
            $6.0$            & 0.5380 & $\pm 0.0103$ \\
            \bottomrule
        \end{tabular}

        \captionof{table}{\textbf{ER-CE outperforms standard CE by $4.4\%$ at $\lambda=1$.} MBPP greedy Pass@1 accuracy under varying entropy penalties,
        aggregated over 5 seeds.}
        \label{tab:mbpp_entropy_main_results}
    \end{minipage}
    \end{adjustbox}
\end{figure}
\label{subsec: code_generation}

\paragraph{Results.} As shown in Figure~\ref{fig:mbpp_entropy_curve}, applying entropy regularization yields a robust, statistically significant improvement over the standard Cross-Entropy baseline ($\lambda = 0.0$). At the baseline ($\lambda = 0.0$), the model achieves a Pass@1 accuracy of $0.5270 \pm 0.0103$. As the entropy penalty increases, performance climbs steeply, peaking at $\lambda = 1.0$ with an absolute gain of $+4.40\%$, reaching $0.5710 \pm 0.0058$. However, continuing to crank up the penalty induces over-sharpening. At $\lambda = 6.0$, the accuracy to degrade back to baseline levels ($0.5380 \pm 0.0103$).

\subsection{Scaling: Harder Problems and a Larger Model}
\label{sec:math7b}

The experiments above use a 1.5B model on an easier benchmark. Two questions remain. First, does the effect survive on a harder benchmark, where the verifier is correspondingly less forgiving? Second, does it survive at a scale where the base model has substantially richer pretrained representations?

We evaluate on MATH \citep{hendrycks2021measuring}, a competition-mathematics benchmark, using \texttt{Qwen2.5-7B-Instruct} \citep{yang2024qwen2}. As before, we fine-tune with LoRA \citep{hu2022lora} across a grid of penalty strengths and evaluate. Correctness is determined by symbolic equivalence to the reference answer rather than string equality, so a model is credited for any expression denoting the correct value.
\begin{figure}[h]
    \centering

    \begin{adjustbox}{valign=t}
    \begin{minipage}{0.52\textwidth}
        \centering
        \includegraphics[width=\linewidth]{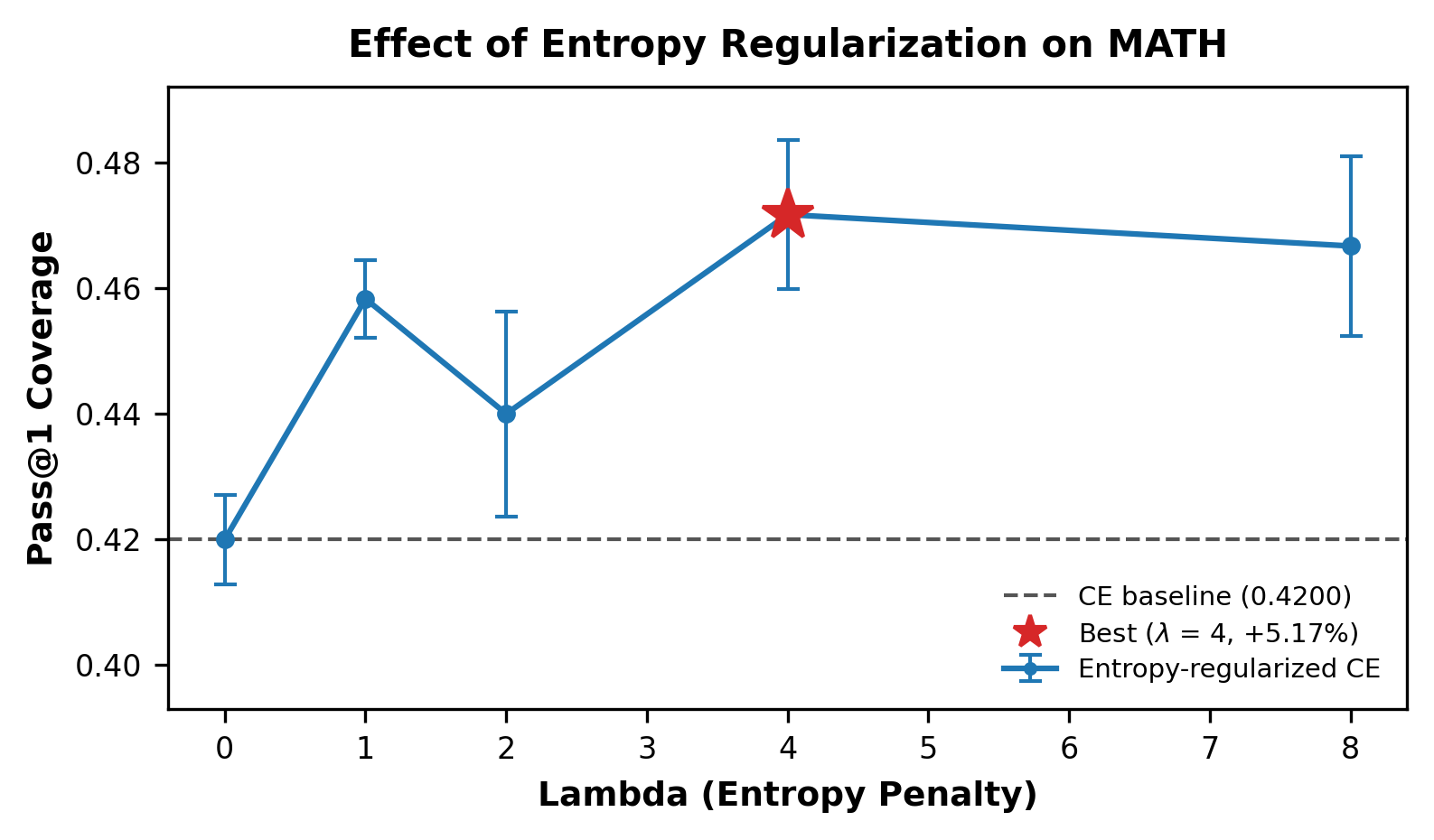}
        \caption{\textbf{Performance gain from ER-CE uniformly remains even for a 7B model.} Greedy Pass@1 accuracy on MATH versus entropy penalty $\lambda$ for
        \texttt{Qwen2.5-7B-Instruct}. Error bars show $\pm 1$ standard deviation over
        3 seeds; shading denotes the baseline interval.}
        \label{fig:math7b}
    \end{minipage}
    \end{adjustbox}
    \hfill
    \begin{adjustbox}{valign=t}
    \begin{minipage}{0.43\textwidth}
        \centering
        \small

        \begin{tabular}{lcc}
            \toprule
            $\lambda$ & \textbf{Mean} & \textbf{Std. Dev.} \\
            \midrule
            $0.0$ (Baseline) & 0.4200 & $\pm 0.0071$ \\
            $1.0$            & 0.4583 & $\pm 0.0062$ \\
            $2.0$            & 0.4400 & $\pm 0.0163$ \\
            $4.0$            & \textbf{0.4717} & $\pm \textbf{0.0118}$ \\
            $8.0$            & 0.4667 & $\pm 0.0143$ \\
            \bottomrule
        \end{tabular}

        \captionof{table}{\textbf{ER-CE peaks over normal CE by $\approx5\%$}. MATH greedy Pass@1 accuracy across entropy penalties,
        aggregated over 3 seeds.}
        \label{tab:math7b}
    \end{minipage}
    \end{adjustbox}
\end{figure}

\paragraph{Results.}
Every regularized configuration outperforms the cross-entropy baseline, and the separation is well outside seed variation: the baseline reaches $0.4200 \pm 0.0071$ while the weakest penalty we tried already reaches $0.4583 \pm 0.0062$. The largest gain is $+5.17\%$ absolute at $\lambda = 4$. The regularization clearly improved accuracy, providing evidence of robustness to scaling.

\section{Conclusion}
We study the mismatch between imitation-based training objectives and verifier accuracy in tasks with multiple correct solutions. Although cross-entropy is a natural objective when the goal is to reproduce an expert distribution, it can be poorly aligned to generate any output accepted by a verifier. We make this mismatch precise through a learning-theoretic counterexample and show that controlling the support of the learned policy can mitigate the resulting failure mode. Motivated by this perspective, we propose entropy-regularized cross-entropy, which uses token-level Shannon entropy as a tractable and differentiable proxy for support size. Our theoretical results identify settings in which this objective improves upon standard cross-entropy minimization, while our experiments on mathematical reasoning and code generation result in improved verifier accuracy with negligible additional training overhead.

Two important directions remain for future work. First, efficiently selecting the regularization strength is an important practical challenge, since the minimum nonzero probability is typically unknown. Understanding how verifier accuracy depends on the regularization strength requires a deeper analysis of the optimization procedure, which is beyond the scope of this work. Second, we consider only Pass@1, the probability that a single model-generated response is correct. More advanced inference-time strategies, such as tree search, majority voting, and Best-of-N sampling, introduce different objectives that may require different post-training methods. An important open question is whether entropy minimization remains beneficial under these strategies, or whether alternative regularization objectives are better suited to them.

\clearpage
\bibliographystyle{plainnat}
\bibliography{references}
\clearpage

\appendix

\section{Proof of Proposition~\ref{prop: counter_example}}
\label{app: pf_counter}
Consider the following concrete example: $\cX=\bN, \cY=\set{1,2,3}$ and $\cG:=\set{1,2}^\cX$. Define the function $u:\cX\to \Delta(\cY)$ by
\[u(\cdot|x)=(1/2,1/2,0),\quad \forall\, x\in \cX.\]
For each $g\in \cG$, define the function $h^g:\cX\to\Delta(\cY)$ by
\begin{align*}
    h^g(\cdot|x)=
    \begin{cases}
        (\frac{5}{8},\frac{1}{8},\frac{1}{4}) &\text{ if }g(x)=1\\
        (\frac{1}{8},\frac{5}{8},\frac{1}{4}) &\text{ if }g(x)=2.
    \end{cases}
\end{align*}
Denote $\cH:=\set{u}\cup \set{h^g}_{g\in \cG}\subset (\Delta(\cY))^\cX$. Let $m\in \bN$ and $N_m\in \bN$ be a constant (that we will fix later). Let $\bD\equiv \text{Unif}([N_m])$ and $h^*=u$. Recall that a sample $\set{(X_i,Y_i)}_{i=1}^m$ is generated as follows: independently, for all $i\in [m]$,
\[X_i\sim \bD \text{ and }Y_i|X_i\sim h^*(\cdot|X_i).\]
This induces a distribution $\bP$ over $(\cX\times \cY)^m$. Denote the event
\[B_m:=\set{X_i \text{ are distinct }\forall\, i\in [m]}.\]
Then,
\begin{align*}
    \bP(B_m) &= 1\times \frac{N_m-1}{N_m}\times \frac{N_m-2}{N_m}\times\cdots \times \frac{N_m-(m-1)}{N_m}\\
    &= \prod_{i=1}^m \left(1-\frac{i}{N_m}\right)\\
    &>\prod_{i=1}^{m-1} e^{-\frac{i}{N_m}-\frac{i^2}{N_m^2}}\\
    &= e^{-\frac{m(m-1)}{2N_m}}e^{-\frac{m(m-1)(2m-1)}{6N_m^2}}.
\end{align*}
For $N_m=\lceil \frac{m(m-1)}{2}\rceil$,
\begin{align*}
    e^{-\frac{m(m-1)}{2N_m}}e^{-\frac{m(m-1)(2m-1)}{6N_m^2}} &\geq e^{-\frac{2m(m-1)}{2m(m-1)}}e^{-\frac{4m(m-1)(2m-1)}{6m^2(m-1)^2}}\\
    &= e^{-1}e^{-\frac{2(2m-1)}{3m(m-1)}}\\
    &\geq e^{-\frac{3}{2}}\\
    &>\frac{1}{8}.
\end{align*}
So, if $N_m=\lceil \frac{m(m-1)}{2}\rceil$, $\bP(B_m)>\frac{1}{8}$. Now, if $B_m$ holds, we receive a sample $S$ with all $X_i$ distinct. So, $\exists\, g_S\in \cG$ s.t. $\forall\, (X_i,Y_i)\in S,\, g_S(X_i)=Y_i$, because $Y_i|X_i\sim u_{X_i}$ and $\text{supp}(u_{X_i})=\set{1,2}$). This gives:
\[\frac{1}{m}\sum_{i=1}^m -\log\left(h^{g_S}(Y_i|X_i)\right) = -\log(5/8)<-\log(1/2)=\frac{1}{m}\sum_{i=1}^m -\log\left(u(Y_i|X_i)\right).\]
If $\cA^{\text{CE}}$ denotes the algorithm that, given a sample $S=\set{(X_i,Y_i)}_{i=1}^m$, returns the empirical risk minimizer (ERM):
\[\arg\min_{h\in \cH}\frac{1}{m}\sum_{i=1}^m -\log(h(Y_i|X_i)),\]
then $\cA^{\text{CE}}(S)\neq u$.
For all $h\in \set{h_g}_{g\in \cG}$, $R(h)=\bE_{X\sim \bD}[h_X(y_3)]=\frac{1}{4}$ and $R(u)=0$. Hence, for every $m$, $\exists\, h^*\in \cH$ and $\bD$ on $\cX$ s.t.
\[\bP\left(R\left(\cA^{\text{CE}}(S)\right)>\inf_{h\in \cH}R(h)+\frac{1}{8}\right) \geq \bP(B_m)>\frac{1}{8},\]
because $\inf_{h\in \cH}R(h)=0$. This shows that $\cA^{\text{CE}}$ does not PAC-learn $\cH$ with the verifier-risk.
\section{Proof of Theorem~\ref{th: modify_ce}}
\label{app: pf_modify_ce}
\subsection{Step 1}
\begin{theorem}
    Suppose $\cH$ satisfies Assumption~\ref{ass: finite_support} and Assumption~\ref{ass: lower_bound}.
    Consider the loss $l:\Delta(\cY)\times \cY\to \re$,
    \[l(p,y) = \frac{1}{\eta}1_{\set{y\notin \text{supp}(p)}}+|\text{supp}(p)|.\]
    Then the algorithm $\cA$ that, for a given sample $S=\set{(X_i,Y_i)}_{i=1}^m$, returns
    \[\cA(S)\in \arg\min_{h\in\cH}\frac{1}{m}\sum_{i=1}^m l(h(\cdot|X_i),Y_i),\]
     PAC-learns $\cH$.
     \label{th: step1}
\end{theorem}
We start with the simple observation that $\inf_{h\in\cH}R(h)=R(h^*)=0$. For each $x\in \cX$,
    \begin{align*}
        \sum_{y\notin \text{supp}(h^*(\cdot|x))}h(y|x) &\leq \left|\text{supp}(h(\cdot|x))\setminus \text{supp}(h^*(\cdot|x))\right|\\
        &= \left|\text{supp}(h^*(\cdot|x))\setminus\text{supp}(h(\cdot|x))\right|+\left|\text{supp}(h(\cdot|x))\right|-\left|\text{supp}(h^*(\cdot|x))\right|.
    \end{align*}
    Now,
    \begin{align*}
    \left|\text{supp}(h^*(\cdot|x))\setminus\text{supp}(h(\cdot|x))\right| &= \sum_{y\notin \text{supp}(h(\cdot|x))}1_{\set{y\in \text{supp}(h^*(\cdot|x))}}\\
        &\leq \sum_{y\notin \text{supp}(h(\cdot|x))}\frac{h^*(\cdot|x)(y)}{\eta} \tag*{\text{($\because\, y\in \text{supp}(h^*(\cdot|x))\implies h^*(y|x)\geq \eta$)}}\\
        &= \frac{1}{\eta}\bP_{Y\sim h^*(\cdot|x)}\left(Y\notin \text{supp}(h(\cdot|x))\right)\\
        &= \bE_{Y\sim h^*(\cdot|x)}\left[\frac{1}{\eta}1_{\set{Y\notin \text{supp}(h(\cdot|x))}}\right].
    \end{align*}
    Let $\bP_{\bD,h^*}$ and $\bE_{\bD,h^*}$ denote the probability and expectation for $X\sim \bD,\, Y|X\sim h^*(\cdot|X)$. We have
    \[R(h)\leq \bE_{\bD,h^*}\left[\frac{1}{\eta}1_{\set{Y\notin \text{supp}(h(\cdot|X))}}\right]+\bE_{\bD}[\left|\text{supp}(h(\cdot|X))\right|-\left|\text{supp}(h^*(\cdot|X))\right|].\]
    Denote $\Tilde{R}(h)=\bE_{\bD,h^*}[l(h(\cdot|X),Y)]$. Then
    \[R(h)\leq \Tilde{R}(h)-\Tilde{R}(h^*).\]
    Let $\hat{h}=\cA(S)$. Then
    \[R(\hat{h})-R(h^*)>\epsilon \implies \Tilde{R}(\hat{h})-\Tilde{R}(h^*)>\epsilon.\]
    We can write
    \begin{align*}
        \Tilde{R}(\hat{h})-\Tilde{R}(h^*) &= \Tilde{R}(\hat{h})-\Tilde{R}^m(\hat{h})+\underbrace{\Tilde{R}^m(\hat{h})-\Tilde{R}^m(h^*)}_{\leq 0}+\Tilde{R}^m(h^*)-\Tilde{R}(h^*)\\
        &\leq 2\sup_{h\in\cH}|\Tilde{R}^m(h)-\Tilde{R}(h)|,
    \end{align*}
    where $\Tilde{R}^m(h)=\frac{1}{m}\sum_{i=1}^m l(h(\cdot|X_i),Y_i)$. Denote $L_h(X_i,Y_i):=l(h(\cdot|X_i),Y_i)$ and let $L_\cH:=\set{(x,y)\mapsto L_h(x,y):h\in\cH}$. Then $|L_\cH|\leq |\cR|<\infty$. Further, for all $x\in\cX,y\in\cY,h\in\cH$,
    \[0\leq L_h(x,y)\leq \frac{1}{\eta}+K.\]
    So,
    \begin{align*}
        \bP(R(\hat{h})-R(h^*)>\epsilon)&\leq \bP(\sup_{h\in\cH}|\Tilde{R}^m(h)-\Tilde{R}(h)|>\epsilon/2)\\
        &= \bP\left(\bigcup_{h\in\cH}\set{\left|\frac{1}{m}\sum_{i=1}^m L_h(X_i,Y_i)-\bE[L_h(X,Y)]\right|>\epsilon/2}\right)\\
        &=\bP\left(\bigcup_{L\in L_\cH}\set{\left|\frac{1}{m}\sum_{i=1}^m L(X_i,Y_i)-\bE[L(X,Y)]\right|>\epsilon/2}\right)\\
        &\leq \sum_{L\in L_\cH}\bP\left(\left|\frac{1}{m}\sum_{i=1}^m L(X_i,Y_i)-\bE[L(X,Y)]\right|>\epsilon/2\right)\\
        &\leq 2|\cR|e^{-\frac{m\epsilon^2}{2\left(\frac{1}{\eta}+K\right)^2}},
    \end{align*}
    where the last inequality follows by Hoeffding's Inequality. Thus, for
    \[m\geq 2\frac{\left(\frac{1}{\eta}+K\right)^2}{\epsilon^2}\log\left(\frac{2|\cR|}{\delta}\right),\quad \bP(R(\hat{h})-R(h^*)>\epsilon)\leq \delta.\]
\subsection{Step 2}
\begin{theorem}
    Suppose $\cH$ satisfies Assumption~\ref{ass: finite_support} and Assumption~\ref{ass: lower_bound}.
    Consider the loss $l_\lambda:\Delta(\cY)\times \cY\to \bar{\re}_+=[0,\infty]$,
    \[l_\lambda(p,y) = -\log(p(y))+\lambda|\text{supp}(p)|.\]
    Then the algorithm $\cA$ that, for a given sample $S=\set{(X_i,Y_i)}_{i=1}^m$, returns
    \[\cA(S)\in \arg\min_{h\in\cH}\frac{1}{m}\sum_{i=1}^m l_{\lambda_m}(h(\cdot|X_i),Y_i),\]
    for a sequence $(\lambda_m)$ s.t. $\lambda_m\uparrow\infty$, PAC-learns $\cH$.
    \label{th: step2}
\end{theorem}
\begin{proof}
    Denote
    \[\hat{h}_m:=\arg\min_{h\in\cH}\frac{1}{m}\sum_{i=1}^m l_{\lambda_m}(h(\cdot|X_i),Y_i).\]
    Recall
    \[R(h)\leq \Tilde{R}(h)-\Tilde{R}(h^*),\]
    where $\Tilde{R}(h)=\frac{1}{\eta}\bP_{\bD,h^*}(Y\notin \text{supp}(h(\cdot|X)))+\bE_\bD[|\text{supp}(h(\cdot|X))|]$.
    So,
    \[\Tilde{R}(\hat{h}_m)-\Tilde{R}(h^*) = \frac{1}{\eta}\bP_{\bD,h^*}(Y\notin \text{supp}(\hat{h}_m(\cdot|X)))+\bE_\bD[|\text{supp}(\hat{h}_m(\cdot|X))|]-\bE_\bD[|\text{supp}(h^*(\cdot|X))|].\]
    Denote $r_h(x):=\text{supp}(h(\cdot|x))$.
    First, observe that
    \[\bP_{\bD,h^*}(Y\notin \text{supp}(\hat{h}_m(\cdot|X))) \leq \sup_{r\in\cR}\left(\bE_{\bD,h^*}[1_{\set{Y\notin r(X)}}]-\frac{1}{m}\sum_{i=1}^m1_{\set{Y_i\notin r(X_i)}}\right),\]
    where we use:
    \[\frac{1}{m}\sum_{i=1}^m1_{\set{Y_i\notin r_{\hat{h}_m}(X_i)}}=0,\]
    because otherwise
    \[\frac{1}{m}\sum_{i=1}^m l(\hat{h}_m(\cdot|X_i),Y_i)=\infty>\frac{1}{m}\sum_{i=1}^m l(h^*(\cdot|X_i),Y_i),\]
    which is a contradiction. Secondly,
    \[\bE_\bD[|\text{supp}(\hat{h}_m(\cdot|X))|]-\bE_\bD[|\text{supp}(h^*(\cdot|X))|] \leq 2\sup_{r\in\cR}\left|\bE_\bD[|r(X)|]-\frac{1}{m}\sum_{i=1}^m |r(X_i)|\right|+\frac{\log(1/\eta)}{\lambda_m},\]
    because
    \begin{align*}
         \frac{1}{m}\sum_{i=1}^m-\log(\hat{h}_m(Y_i|X_i))+\lambda_m|\text{supp}(\hat{h}_m(\cdot|X_i))| &=  \frac{1}{m}\sum_{i=1}^m l_{\lambda_m}(\hat{h}_m(\cdot|X_i),Y_i)\\
         &\leq\frac{1}{m}\sum_{i=1}^m l_{\lambda_m}(h^*(\cdot|X_i),Y_i)\\
         &=  \frac{1}{m}\sum_{i=1}^m-\log(h^*(Y_i|X_i))+\lambda_m|\text{supp}(h^*(\cdot|X_i))| 
    \end{align*}
    and the fact that
    \[\hat{h}_m(Y_i|X_i)\leq 1 \text{ and }h^*(Y_i|X_i)\geq \eta.\]
    Finally, we get
    \begin{align*}
        \Tilde{R}(\hat{h}_m)-\Tilde{R}(h^*) &\leq \frac{1}{\eta}\sup_{r\in\cR}\left(\bE_{\bD,h^*}[1_{\set{Y\notin r(X)}}]-\frac{1}{m}\sum_{i=1}^m1_{\set{Y_i\notin r(X_i)}}\right)\\
        &\quad +2\sup_{r\in\cR}\left|\bE_\bD[|r(X)|]-\frac{1}{m}\sum_{i=1}^m |r(X_i)|\right|+\frac{\log(1/\eta)}{\lambda_m}
    \end{align*}
    Let $\epsilon>0$. If  $\lambda_m>\frac{3}{\epsilon}\log(1/\eta)$, then
    \begin{align*}
        \bP(R(\hat{h}_m)-R(h^*)>\epsilon)&\leq \bP\left(\sup_{r\in\cR}\left(\bE_{\bD,h^*}[1_{\set{Y\notin r(X)}}]-\frac{1}{m}\sum_{i=1}^m1_{\set{Y_i\notin r(X_i)}}\right)\geq \frac{\eta\epsilon}{3}\right)\\
        &\quad +\bP\left(\sup_{r\in\cR}\left|\bE_\bD[|r(X)|]-\frac{1}{m}\sum_{i=1}^m |r(X_i)|\right|>\frac{\epsilon}{6}\right).
    \end{align*}
    By Hoeffding's Inequality and the fact that $|r(x)|\leq \frac{1}{\eta}$ for all $x\in \cX,r\in\cR$,
    \[\bP(R(\hat{h}_m)-R(h^*)>\epsilon)\leq 3|\cR|e^{-\frac{m\eta^2\epsilon^2}{18}}.\]
    Finally, for any $\delta>0$, if  $m\geq \max\left\{\frac{18}{\eta^2\epsilon^2}\log\left(\frac{3|\cR|}{\delta}\right),\inf\left\{m\in\bN:\lambda_m>\frac{3}{\epsilon}\log\left(\frac{1}{\eta}\right)\right\}\right\}$,
    \[\bP(R(\hat{h}_m)-R(h^*)>\epsilon)\leq \delta.\]
\end{proof}
\begin{corollary}
    If we set $\lambda_m=m$, then the algorithm $\cA$ that, for a given sample $S=\set{(X_i,Y_i)}_{i=1}^m$, returns
    \[\cA(S)\in \arg\min_{h\in\cH}\frac{1}{m}\sum_{i=1}^m l_{\lambda_m}(h(\cdot|X_i),Y_i),\]
    PAC-learns $\cH$, with sample complexity
    \[m(\epsilon,\delta)=\left\lceil \max\left\{\frac{18}{\eta^2\epsilon^2}\log\left(\frac{3|\cR|}{\delta}\right),\frac{3}{\epsilon}\log\left(\frac{1}{\eta}\right)\right\} \right\rceil.\]
\end{corollary}
\subsection{Step 3}
 We need a result connecting $S_\alpha(\cdot)$ and $|\text{supp}(\cdot)|$ for $\alpha$ close to $0$.
\begin{lemma}
    Let $0<\alpha<1$. If $p\in \Delta(\cY)$ s.t. $p(y)\in \set{0}\cup [\eta,1]$, then
    \[0\leq |\text{supp}(p)|-S_\alpha(p) \leq \frac{1}{\eta}\left(1-\eta^{\frac{\alpha}{1-\alpha}}\right).\]
    \label{lem: S_alpha_supp}
\end{lemma}
\begin{proof}
    First, we show that $S_\alpha(p)$ is non-increasing in $\alpha$, so that if $\alpha>0$,
    \[|\text{supp}(p)|=S_0(p)\geq S_\alpha(p).\]
    For $0\leq\alpha<\beta< 1$, let
    \[Z_\alpha = \sum_y p(y)^\alpha \text{ and }Z_\beta = \sum_y p(y)^\beta.\]
   with the convention that $Z_0:=|\text{supp}(p)|$. Then
    \[Z_\beta = Z_\alpha\sum_y \frac{p(y)^\alpha}{Z_\alpha}\left(p(y)^{1-\alpha}\right)^{\frac{\beta-\alpha}{1-\alpha}}.\]
    Now, $f(x)=x^{\frac{\beta-\alpha}{1-\alpha}}$ is concave, so,
    \[Z_\beta \leq Z_\alpha \left(\sum_y \frac{p(y)^\alpha}{Z_\alpha}p(y)^{1-\alpha}\right)^{\frac{\beta-\alpha}{1-\alpha}} = \frac{Z_\alpha}{Z_\alpha^{\frac{\beta-\alpha}{1-\alpha}}}=Z_\alpha^{\frac{1-\beta}{1-\alpha}}.\]
    Raising both sides to the power of $\frac{1}{1-\beta}$ shows the claim.\\
    To show the other direction:
    \[S_\alpha(p) = \left(\sum_y p(y)^\alpha\right)^{\frac{1}{1-\alpha}}\geq \left(\sum_{y\in\text{supp}(p)}\eta^\alpha\right)^{\frac{1}{1-\alpha}} =|\text{supp}(p)|^{\frac{1}{1-\alpha}}\eta^{\frac{\alpha}{1-\alpha}} = |\text{supp}(p)|\left(|\text{supp}(p)|\eta\right)^{\frac{\alpha}{1-\alpha}}\geq |\text{supp}(p)|\eta^{\frac{\alpha}{1-\alpha}},\]
    where the last inequality follows from $|\text{supp}(p)|\geq 1$. Since $|\text{supp}(p)|\leq \frac{1}{\eta}$, we get
    \[|\text{supp}(p)|-S_\alpha(p)\leq |\text{supp}(p)|\left(1-\eta^{\frac{\alpha}{1-\alpha}}\right)\leq \frac{1-\eta^{\frac{\alpha}{1-\alpha}}}{\eta}.\]
\end{proof}
Now, we're ready to prove the most general result.
\begin{theorem}
    Suppose $\cH$ satisfies Assumption~\ref{ass: finite_support} and Assumption~\ref{ass: lower_bound}.
    Consider the loss $l_{\lambda,\alpha}:\Delta(\cY)\times \cY\to \bar{\re}_+=[0,\infty]$,
    \[l_{\lambda,\alpha}(p,y) = -\log(p(y))+\lambda S_\alpha(p).\]
    Then the algorithm $\cA$ that, for a given sample $S=\set{(X_i,Y_i)}_{i=1}^m$, returns
    \[\cA(S)\in \arg\min_{h\in\cH}\frac{1}{m}\sum_{i=1}^m l_{\lambda_m,\alpha_m}(h(\cdot|X_i),Y_i),\]
    for sequences $(\lambda_m)$ and $(\alpha_m)$ s.t. $\lambda_m\uparrow\infty$ and $0<\alpha_m<1,\, \alpha_m\downarrow 0$, PAC-learns $\cH$.
    \label{th: step3}
\end{theorem}
\begin{proof}
    Denote
    \[\hat{h}_m:=\arg\min_{h\in\cH}\frac{1}{m}\sum_{i=1}^m l_{\lambda_m,\alpha_m}(h(\cdot|X_i),Y_i).\]
    Similar to the proof for $l_\lambda(p,y)$, we can write:
    \[\Tilde{R}(\hat{h}_m)-\Tilde{R}(h^*) = \frac{1}{\eta}\underbrace{\bP_{\bD,h^*}(Y\notin \text{supp}(\hat{h}_m(\cdot|X)))}_{\mathbf{(1)}}+\underbrace{\bE_\bD[|\text{supp}(\hat{h}_m(\cdot|X))|]-\bE_\bD[|\text{supp}(h^*(\cdot|X))|]}_{\mathbf{(2)}}.\]
    The analysis for $\mathbf{(1)}$ is the same as before and we get:
    \[\bP_{\bD,h^*}(Y\notin \text{supp}(\hat{h}_m(\cdot|X))) \leq \sup_{r\in\cR}\left(\bE_{\bD,h^*}[1_{\set{Y\notin r(X)}}]-\frac{1}{m}\sum_{i=1}^m1_{\set{Y_i\notin r(X_i)}}\right).\]
    For $\mathbf{(2)}$, we first write
    \[\bE_\bD[|\text{supp}(\hat{h}_m(\cdot|X))|]-\bE_\bD[|\text{supp}(h^*(\cdot|X))|] \leq 2\sup_{r\in\cR}\left|\bE_\bD[|r(X)|]-\frac{1}{m}\sum_{i=1}^m |r(X_i)|\right|+\frac{1}{m}\sum_{i=1}^m |r_{\hat{h}_m}(X_i)|-\frac{1}{m}\sum_{i=1}^m |r_{h^*}(X_i)|.\]
    Observe that
    \begin{align*}
        \frac{1}{m}\sum_{i=1}^m |r_{\hat{h}_m}(X_i)|-\frac{1}{m}\sum_{i=1}^m |r_{h^*}(X_i)| &\leq \frac{\log(1/\eta)}{\lambda_m}+\frac{1}{m}\sum_{i=1}^m \left(|r_{\hat{h}_m}(X_i)|-S_\alpha(\hat{h}_m(\cdot|X_i))\right)\\
        &\quad+\frac{1}{m}\sum_{i=1}^m \left(S_\alpha(h^*(\cdot|X_i))-|r_{h^*}(X_i)|\right).
    \end{align*}
    Using Lemma~\ref{lem: S_alpha_supp}, we obtain
    \[\frac{1}{m}\sum_{i=1}^m |r_{\hat{h}_m}(X_i)|-\frac{1}{m}\sum_{i=1}^m |r_{h^*}(X_i)| \leq \frac{\log(1/\eta)}{\lambda_m}+\frac{\left(1-\eta^{\frac{\alpha_m}{1-\alpha_m}}\right)}{\eta}.\]
    Finally,
    \begin{align*}
        \Tilde{R}(\hat{h}_m)-\Tilde{R}(h^*) &\leq \frac{1}{\eta}\sup_{r\in\cR}\left(\bE_{\bD,h^*}[1_{\set{Y\notin r(X)}}]-\frac{1}{m}\sum_{i=1}^m1_{\set{Y_i\notin r(X_i)}}\right)\\
        &\quad+2\sup_{r\in\cR}\left|\bE_\bD[|r(X)|]-\frac{1}{m}\sum_{i=1}^m |r(X_i)|\right|+\frac{\log(1/\eta)}{\lambda_m}+\frac{\left(1-\eta^{\frac{\alpha_m}{1-\alpha_m}}\right)}{\eta}.
    \end{align*}
    Let $\epsilon>0$. If $\lambda_m>\frac{4}{\epsilon}\log(1/\eta)$ and $\frac{\left(1-\eta^{\frac{\alpha_m}{1-\alpha_m}}\right)}{\eta}<\frac{\epsilon}{4}$, 
    \begin{align*}
        \bP(R(\hat{h}_m)-R(h^*)>\epsilon)&\leq \bP\left(\sup_{r\in\cR}\left(\bE_{\bD,h^*}[1_{\set{Y\notin r(X)}}]-\frac{1}{m}\sum_{i=1}^m1_{\set{Y_i\notin r(X_i)}}\right)\geq \frac{\eta\epsilon}{4}\right)\\
        &\quad +\bP\left(\sup_{r\in\cR}\left|\bE_\bD[|r(X)|]-\frac{1}{m}\sum_{i=1}^m |r(X_i)|\right|>\frac{\epsilon}{8}\right).
    \end{align*}
    By Hoeffding's Inequality,
    \[\bP(R(\hat{h}_m)-R(h^*)>\epsilon)\leq 3|\cR|e^{-\frac{m\eta^2\epsilon^2}{32}}.\]
    Finally, for any $\delta>0$, if  
    \[m\geq \max\left\{\frac{32}{\eta^2\epsilon^2}\log\left(\frac{3|\cR|}{\delta}\right),\inf\left\{m\in\bN:\lambda_m>\frac{4}{\epsilon}\log\left(\frac{1}{\eta}\right)\right\},\inf\left\{m\in\bN:\frac{\left(1-\eta^{\frac{\alpha_m}{1-\alpha_m}}\right)}{\eta}<\frac{\epsilon}{4}\right\}\right\},\]
    then
    \[\bP(R(\hat{h}_m)-R(h^*)>\epsilon)\leq \delta.\]
\end{proof}
\begin{corollary}
    If we set $\lambda_m=m$ and $\alpha=1/m$, then the algorithm $\cA$ that, for a given sample $S=\set{(X_i,Y_i)}_{i=1}^m$, returns
    \[\cA(S)\in \arg\min_{h\in\cH}\frac{1}{m}\sum_{i=1}^m l_{\lambda_m,\alpha_m}(h(\cdot|X_i),Y_i),\]
    PAC-learns $\cH$, with sample complexity
    \[m(\epsilon,\delta)=\left\lceil \max\left\{\frac{32}{\eta^2\epsilon^2}\log\left(\frac{3|\cR|}{\delta}\right),\frac{4}{\epsilon}\log\left(\frac{1}{\eta}\right),1+\frac{\log\eta}{\log\left(1-\frac{\eta\epsilon}{4}\right)}\right\} \right\rceil.\]
\end{corollary}
\clearpage
\section{Discussion on $(S_\alpha(\cdot))$}
\label{app: discussion_s_alpha}
For practical purposes, the support penalty is computationally tedious to deal with; it's non-differentiable and extremely sensitive to precision limits. We desire a surrogate $S(\cdot)$ for the support that satisfies certain obvious properties:
\begin{enumerate}
    \item $S(\cdot)$ should be continuous and symmetric, i.e. invariant under reordering of labels.
    \item $S(\delta_y)=1\leq S(p)\leq S(\text{Unif}(\cY))=|\cY|$ for any distribution $p$ on $\cY$, where $\delta_y$ places unit mass on $y\in\cY$ and $\text{Unif}(\cY)$ is the uniform distribution on $\cY$.
    \item For any distribution $p$ on $\cY$, $S(p)=S((p,0))$, where $(p,0)$ is the distribution on $\cY\cup\set{*}$ with weight $p(y)$ on $y\in\cY$ and weight $0$ on $*$.
    \item For distributions $p_1,p_2$, $S(q)=S(p_1)S(p_2)$, where $q=p_1\otimes p_2$ is the product distribution of $p_1$ and $p_2$.
\end{enumerate}
The family of functions $\set{S_\alpha(\cdot)}_{\alpha>0,\alpha\neq 1 }$ given by
$S_\alpha(p) = \left(\sum_{y\in\cY}p^\alpha(y)\right)^{\frac{1}{1-\alpha}}$ for $p\in \Delta(\cY)$ satisfies all the properties above \citep{e8030169}. Observe that
\begin{itemize}
    \item For $\alpha>0$ and $\alpha\neq 1$, $\log(S_\alpha(\cdot))=R_\alpha(\cdot)$, where $R_\alpha(p)=\frac{1}{1-\alpha}\log\left(\sum_{y\in\cY}p^\alpha(y)\right)$ is the Renyi entropy of order $\alpha$.
    \item In the limit $\alpha\to 1$, $\log(S_\alpha(\cdot))\to H(\cdot)$, where $H(p)=-\sum_{y\in\cY}p(y)\log(p(y))$ is the Shannon entropy.
    \item In the limit $\alpha\to 0$, $\log(S_\alpha(p))\to \log(|\text{supp}(\cdot)|)$, the logarithm of the size of the support.
\end{itemize}
\clearpage
\section{Experimental Setup GSM8k}
\label{sec:appendix_setup}

To ensure full reproducibility, we detail the hyperparameters, data processing, and evaluation metrics used in our empirical validation.

\subsection{Data Preparation}
We use the \textbf{GSM8K} dataset, a standard benchmark for mathematical reasoning.
\begin{itemize}
    \item \textbf{Train Split:} 1,000 randomly sampled problems.
    \item \textbf{Evaluation Split:} 500 randomly sampled problems.
    \item \textbf{Formatting:} The prompt explicitly instructs the model to "Solve this step by step" and end the response with the format \texttt{\#\#\#\# <answer>}. Calculator annotations (e.g., \texttt{<<...>>}) are stripped from the target solutions.
    \item \textbf{Masking:} Prompt tokens are masked (label = $-100$) so that both the cross-entropy loss and the entropy penalty are computed strictly over the generated response tokens.
\end{itemize}

\subsection{Model Architecture \& LoRA Configuration}
Experiments are conducted using the \textbf{Qwen/Qwen2.5-1.5B-Instruct} base model. Due to hardware constraints and to ensure training stability, we utilize Low-Rank Adaptation (LoRA).
\begin{itemize}
    \item \textbf{Precision:} bfloat16 (\texttt{bf16}) with Flash Attention 2 \cite{dao2023flashattention2}.
    \item \textbf{LoRA Rank ($r$):} 16
    \item \textbf{LoRA Alpha:} 32
    \item \textbf{LoRA Dropout:} 0.05
    \item \textbf{Target Modules:} \texttt{q\_proj}, \texttt{k\_proj}, \texttt{v\_proj}, \texttt{o\_proj}, \texttt{gate\_proj}, \texttt{up\_proj}, \texttt{down\_proj}.
\end{itemize}

\subsection{Training Hyperparameters}
Models are trained using a custom Hugging Face Trainer \cite{wolf-etal-2020-transformers} implementation with PEFT \cite{peft} that calculates the in-graph token entropy dynamically.
\begin{itemize}
    \item \textbf{Learning Rate:} $2 \times 10^{-4}$
    \item \textbf{LR Scheduler:} Cosine decay with a 0.03 warmup ratio.
    \item \textbf{Epochs:} 3.0
    \item \textbf{Effective Batch Size:} 16 (8 per device $\times$ 2 gradient accumulation steps).
    \item \textbf{Optimizer:} AdamW \cite{loshchilov2018decoupled} (max gradient norm clipped to 1.0).
    \item \textbf{Sequence Lengths:} Maximum sequence length of 768 tokens; max new generated tokens capped at 512 for CoT.
\end{itemize}

\subsection{Evaluation details}
We utilize a pure Greedy Pass@1 evaluation metric. The model performs deterministic generation (argmax decoding) independently of temperature. Answer correctness is verified via strict string matching and float tolerance ($< 10^{-4}$) after stripping non-numeric characters from the final boxed output. To account for variance in LoRA initialization and data shuffling, every $\lambda$ configuration is run across 5 distinct random seeds ($0 \dots 4$), with the mean and standard deviation reported.

\begin{table}[h]
    \centering
    \caption{Detailed per-seed Pass@1 Accuracy on GSM8K across varying entropy penalty ($\lambda$) strengths using the Qwen2.5-1.5B-Instruct model.}
    \label{tab:entropy_appendix_results}
    \resizebox{\textwidth}{!}{%
    \begin{tabular}{l ccccc cc}
        \toprule
        & \multicolumn{5}{c}{\textbf{Individual Seeds (Pass@1)}} & \multicolumn{2}{c}{\textbf{Aggregate}} \\
        \cmidrule(lr){2-6} \cmidrule(lr){7-8}
        \textbf{Entropy Penalty ($\lambda$)} & \textbf{Seed 0} & \textbf{Seed 1} & \textbf{Seed 2} & \textbf{Seed 3} & \textbf{Seed 4} & \textbf{Mean} & \textbf{Std. Dev.} \\
        \midrule
        $0.0$ (Baseline) & 0.5080 & 0.4620 & 0.4960 & 0.4940 & 0.4960 & 0.4912 & $\pm 0.0154$ \\
        $0.3$            & 0.4840 & 0.5140 & 0.5040 & 0.5100 & 0.5160 & 0.5056 & $\pm 0.0116$ \\
        $0.5$            & 0.5020 & 0.5100 & 0.5100 & 0.5000 & 0.5180 & 0.5080 & $\pm 0.0064$ \\
        $1.0$            & 0.5220 & 0.5180 & 0.5320 & 0.5340 & 0.5320 & 0.5276 & $\pm 0.0064$ \\
        $2.0$            & 0.5260 & 0.5500 & 0.5440 & 0.5420 & 0.5320 & 0.5388 & $\pm 0.0086$ \\
        $3.0$            & 0.5440 & 0.5580 & 0.5400 & 0.5480 & 0.5620 & 0.5504 & $\pm 0.0083$ \\
        $4.0$            & 0.5460 & 0.5540 & 0.5380 & 0.5780 & 0.5520 & 0.5536 & $\pm 0.0134$ \\
        $6.0$            & 0.5640 & 0.5600 & 0.5960 & 0.5700 & 0.5840 & 0.5748 & $\pm 0.0134$ \\
        $8.0$            & 0.6020 & 0.5900 & 0.5540 & 0.5760 & 0.5860 & 0.5816 & $\pm 0.0161$ \\
        \bottomrule
    \end{tabular}%
    }
\end{table}

\clearpage
\section{Experimental Setup MBPP}
\label{sec:appendix_mbpp_setup}

To ensure full reproducibility, we detail the hyperparameters, data processing, and execution-based evaluation used for the MBPP code generation experiments.

\subsection{Data Preparation}
We use the sanitized split of the \textbf{MBPP} dataset \cite{austin2021program}, which ensures high-quality problem descriptions and correct unit tests.
\begin{itemize}
    \item \textbf{Train Split:} 300 randomly sampled problems.
    \item \textbf{Evaluation Split:} 200 randomly sampled problems.
    \item \textbf{Formatting:} The prompt explicitly instructs the model to write a complete Python function and provides the first unit-test as a hint to enforce proper function naming. The model is required to return the code enclosed in standard Markdown Python blocks (\texttt{```python ... ```}).
    \item \textbf{Masking:} Prompt tokens are masked (label = $-100$) so that both the cross-entropy loss and the entropy penalty are computed strictly over the generated Python AST (Abstract Syntax Tree) tokens.
\end{itemize}

\subsection{Model Architecture \& LoRA Configuration}
Experiments are conducted using the \textbf{Qwen/Qwen2.5-1.5B-Instruct} base model, utilizing Low-Rank Adaptation (LoRA) for training stability.
\begin{itemize}
    \item \textbf{Precision:} bfloat16 (\texttt{bf16}) with scaled dot-product attention (SDPA).
    \item \textbf{LoRA Rank ($r$):} 16
    \item \textbf{LoRA Alpha:} 32
    \item \textbf{LoRA Dropout:} 0.05
    \item \textbf{Target Modules:} \texttt{q\_proj}, \texttt{k\_proj}, \texttt{v\_proj}, \texttt{o\_proj}, \texttt{gate\_proj}, \texttt{up\_proj}, \texttt{down\_proj}.
\end{itemize}

\subsection{Training Hyperparameters}
Models are trained using a custom Hugging Face Trainer \cite{wolf-etal-2020-transformers} implementation with PEFT \cite{peft}.
\begin{itemize}
    \item \textbf{Learning Rate:} $2 \times 10^{-4}$
    \item \textbf{LR Scheduler:} Cosine decay with a 0.03 warmup ratio.
    \item \textbf{Epochs:} 3.0
    \item \textbf{Effective Batch Size:} 16 (8 per device $\times$ 2 gradient accumulation steps).
    \item \textbf{Optimizer:} AdamW \cite{loshchilov2018decoupled} (max gradient norm clipped to 1.0).
    \item \textbf{Sequence Lengths:} Maximum sequence length of 768 tokens; max new generated tokens capped at 256 to accommodate multiline code structures.
\end{itemize}

\subsection{Execution-Based Evaluation}
Because code correctness cannot be determined by static string matching, we utilize a dynamic evaluation pipeline. We employ a pure Greedy Pass@1 evaluation metric, wherein the model performs deterministic generation (argmax decoding, $T=0.0$). 

The extracted code blocks are passed to a sandboxed Python \texttt{exec()} environment alongside the suite of dataset-provided \texttt{assert} statements. To prevent hanging from model-generated infinite loops (e.g., non-terminating \texttt{while} conditions), we wrap the execution environment in a UNIX signal alarm, enforcing a strict 2.0-second timeout per problem. Code that throws any exception (\texttt{AssertionError}, \texttt{SyntaxError}, \texttt{NameError}) or exceeds the timeout is marked incorrect. Every $\lambda$ configuration is run across 5 distinct random seeds ($0 \dots 5$) to capture variance in LoRA initialization and training data order.

\begin{table}[h]
    \centering
    \caption{Detailed per-seed Greedy Pass@1 Accuracy on MBPP across varying entropy penalty ($\lambda$) strengths using the Qwen2.5-1.5B-Instruct model.}
    \label{tab:mbpp_entropy_appendix_results}
    \begin{tabular}{l ccccc cc}
        \toprule
        & \multicolumn{5}{c}{\textbf{Individual Seeds}} & \multicolumn{2}{c}{\textbf{Aggregate}} \\
        \cmidrule(lr){2-6} \cmidrule(lr){7-8}
        \textbf{Entropy Penalty ($\lambda$)} & \textbf{Seed 0} & \textbf{Seed 1} & \textbf{Seed 2} & \textbf{Seed 3} & \textbf{Seed 4} & \textbf{Mean} & \textbf{Std. Dev.} \\
        \midrule
        $0.0$ (Baseline) & 0.5450 & 0.5300 & 0.5150 & 0.5200 & 0.5250 & 0.5270 & $\pm 0.0103$ \\
        $0.3$            & 0.5300 & 0.5550 & 0.5350 & 0.5200 & 0.5500 & 0.5380 & $\pm 0.0129$ \\
        $0.7$            & 0.5700 & 0.5700 & 0.5600 & 0.5800 & 0.5500 & 0.5660 & $\pm 0.0102$ \\
        $1.0$            & 0.5750 & 0.5700 & 0.5750 & 0.5600 & 0.5750 & 0.5710 & $\pm 0.0058$ \\
        $1.5$            & 0.5600 & 0.5600 & 0.5550 & 0.5400 & 0.5550 & 0.5540 & $\pm 0.0073$ \\
        $2.0$            & 0.5500 & 0.5450 & 0.5550 & 0.5550 & 0.5400 & 0.5490 & $\pm 0.0058$ \\
        $3.0$            & 0.5400 & 0.5400 & 0.5300 & 0.5300 & 0.5300 & 0.5340 & $\pm 0.0049$ \\
        $6.0$            & 0.5250 & 0.5400 & 0.5400 & 0.5550 & 0.5300 & 0.5380 & $\pm 0.0103$ \\
        \bottomrule
    \end{tabular}
\end{table}

\clearpage

\section{Experimental Setup: MATH}
\label{app:math}

We detail the data processing, hyperparameters, answer-equivalence procedure, and the loss implementation used for the MATH experiments of Section~\ref{sec:math7b}.

\subsection{Data Preparation}

We use the MATH dataset \citep{hendrycks2021measuring}.

\begin{itemize}
\item \textbf{Train Split:} 1{,}000 problems sampled uniformly at random (seed 42) from the MATH training split, pooled across all seven subject categories.
\item \textbf{Evaluation Split:} 200 problems sampled from MATH-500, a held-out subset of the MATH test split. Evaluating on a test-split subset guarantees no overlap with the fine-tuning data.
\item \textbf{Filtering:} Training problems whose reference solution contains no \verb|\boxed{}| expression are discarded, since there is no answer to imitate. Problems containing Asymptote figures (\verb|[asy]|) are also discarded, to avoid fine-tuning the model to emit diagram markup.
\item \textbf{Formatting:} The prompt instructs the model to solve the problem step by step and to place its final answer in \verb|\boxed{}|. Reference solutions are used verbatim as targets; they already terminate in a boxed answer.
\item \textbf{Masking:} Prompt tokens are masked ($\text{label} = -100$) so that both the cross-entropy loss and the entropy penalty are computed strictly over response tokens.
\end{itemize}

\subsection{Model Architecture and LoRA Configuration}

Experiments use \texttt{Qwen/Qwen2.5-7B-Instruct} \citep{yang2024qwen2} with Low-Rank Adaptation \citep{hu2022lora}.

\begin{itemize}
\item \textbf{Precision:} bfloat16 (bf16) with FlashAttention-2 \citep{dao2023flashattention2}.
\item \textbf{LoRA Rank ($r$):} 16
\item \textbf{LoRA Alpha:} 32
\item \textbf{LoRA Dropout:} 0.05
\item \textbf{Target Modules:} \texttt{q\_proj}, \texttt{k\_proj},
\texttt{v\_proj}, \texttt{o\_proj}, \texttt{gate\_proj}, \texttt{up\_proj},
\texttt{down\_proj}.
\end{itemize}

\subsection{Training Hyperparameters}

Models are trained with a custom Hugging Face \texttt{Trainer} \citep{wolf-etal-2020-transformers} implementation using PEFT \citep{peft} that computes the in-graph token entropy dynamically.

\begin{itemize}
\item \textbf{Learning Rate:} $2 \times 10^{-4}$
\item \textbf{LR Scheduler:} Cosine decay with a 0.03 warmup ratio.
\item \textbf{Epochs:} 3.0
\item \textbf{Effective Batch Size:} 16 (2 per device $\times$ 8 gradient accumulation steps), with gradient checkpointing enabled.
\item \textbf{Optimizer:} AdamW \citep{loshchilov2018decoupled} (max gradient norm clipped to 1.0).
\item \textbf{Sequence Lengths:} Maximum sequence length of 1{,}024 tokens; max new generated tokens capped at 1{,}024, reflecting the longer derivations MATH requires relative to GSM8K.
\item \textbf{Seeds:} Every $\lambda$ configuration is run across 3 random seeds ($0 \ldots 2$), with mean and standard deviation reported.
\end{itemize}

\subsection{Chunked Computation of the Entropy Term}
\label{app:chunking}

The entropy penalty requires a full softmax over the vocabulary at every response position. Computed in fp32, as numerical stability demands, this is a tensor of shape $N \times V$ with $V = 151{,}936$ for Qwen2.5, where $N$ is the number of response tokens in the batch; several such tensors are live simultaneously during the backward pass. At 7B this footprint is prohibitive alongside the model weights.

We therefore evaluate the loss in chunks of 512 response positions, wrapping each chunk in gradient checkpointing so that the fp32 intermediates are discarded after the forward pass and recomputed during backward. Cross-entropy and entropy are accumulated as sums over chunks and normalized once at the end. The computation is numerically identical to the unchunked form; peak memory for the loss head falls from $O(NV)$ to $O(CV)$ for chunk size $C$. This costs one additional forward pass through the softmax and adds negligible wall-clock time relative to the transformer forward.

\subsection{Answer Equivalence}
\label{app:equiv}

Unlike GSM8K, MATH answers are arbitrary \LaTeX{} expressions ($\frac{3}{4}$, $2\sqrt{3}$, $(-\infty, 2]$, $x^2 + 1$), so string equality would systematically under-count correct answers written in a different but equivalent form. We use the following procedure.

\begin{enumerate}
\item \textbf{Extraction.} The last \verb|\boxed{}| expression in the generation is extracted by brace matching. If none is present, an explicit ``the final answer is $X$'' pattern is accepted as a fallback. We deliberately do \emph{not} fall back to the last numeral in the generation: on MATH this manufactures spurious correctness whenever a digit appearing in the reasoning trace coincides with the reference answer, and would bias the comparison toward whichever configuration produces shorter derivations.
\item \textbf{Normalization.} The extracted answer and the reference are each normalized under two independent conventions, covering fraction and radical canonicalization, removal of units, degree symbols, currency symbols and \verb|\text{}| wrappers, and digit-group separators. A match under either convention is accepted.
\item \textbf{Numeric comparison.} If both normalized forms parse as reals, they are compared with tolerance $10^{-6}$.
\item \textbf{Symbolic comparison.} As a last resort, both expressions are converted to symbolic form and their difference simplified; a difference of zero is accepted. Each such check is guarded by a 5-second timeout, since model output can induce pathological simplification.
\end{enumerate}

\clearpage

\section{Temperature Ablation}
\label{app:sampled}

Every result in the main paper is measured with greedy Pass@1. This appendix reports the same three experiments measured instead with temperature-sampled Pass@1, explains why sampled evaluation is the closer empirical analogue of the verifier risk our theory is stated in terms of, and explains why we nonetheless promote the greedy numbers to the main paper.

\subsection{Why Sampled Evaluation is Closer to the Theory}

The verifier risk of Section~3 is $R(h) = \mathbb{E}_{X \sim \mathcal{D},\, Y \sim h(\cdot|X)}[\mathbf{1}\{Y \notin \mathcal{C}_X\}]$: an expectation over draws from the policy, equal to the probability mass the policy places outside the correct set. Sampling $k$ completions per problem and scoring the fraction that the verifier accepts is a direct Monte Carlo estimate of $1 - R(h)$ for the decoded policy. Greedy Pass@1 is not an estimate of this quantity at all. It is a property of a single argmax trajectory, and it is almost entirely insensitive to how mass is allocated away from that trajectory, which is precisely the degree of freedom Proposition~3.2 exploits and the entropy penalty is designed to control. Two policies that differ enormously in verifier risk can decode identically under argmax.

Sampled evaluation therefore tests the mechanism, not just the outcome. By withdrawing mass from incorrect outputs rather than by some incidental effect on the argmax path, the measured gains should be \emph{larger} under sampling than under greedy decoding. They are, on all three benchmarks.

One caveat on the match: sampling at $T = 0.8$ estimates the verifier risk of the \emph{tempered} policy $h^{1/T}$, not of $h$ itself. Only $T = 1$ recovers $R(h)$ exactly. We report $T = 0.8$ because it is the decoding temperature in common practical use.

\subsection{Protocol}

Training is unchanged in every respect, same model, same LoRA configuration, same hyperparameters, same seeds, same $\lambda$ grid, same data. Only decoding differs. For each problem we draw $k$ independent samples at temperature $T = 0.8$ and report
\[
\text{sampled Pass@1} \;=\; \frac{\text{number of accepted samples}}
{n_{\text{eval}} \cdot k}.
\]
We use $k = 8$ on GSM8K and $k = 4$ on MBPP and MATH. Sampling is from the \emph{full} tempered softmax: we set \texttt{top\_k} $= 0$ and \texttt{top\_p} $= 1$. Truncated sampling discards the tail of the predictive distribution, which is exactly the region whose mass this paper is about; measuring under truncation would suppress both the failure mode and the correction. The GSM8K sampled sweep uses a coarser $\lambda$ grid and three seeds rather than five; MBPP and MATH use the full main-paper grids and seed counts.

\subsection{Results}

\begin{figure}[h]
\centering
\includegraphics[width=0.62\textwidth]{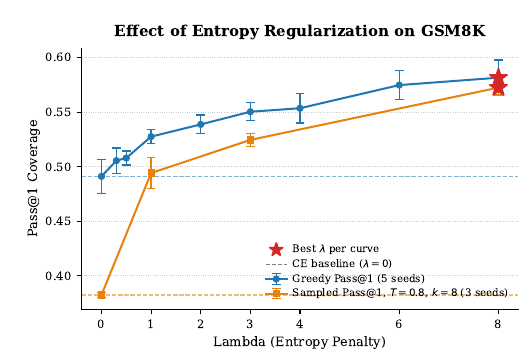}
\caption{GSM8K Pass@1 under greedy decoding and under temperature sampling ($T = 0.8$, $k = 8$) versus entropy penalty $\lambda$. Dashed lines mark each curve's cross-entropy baseline; stars mark each curve's best $\lambda$. Error bars show $\pm 1$ standard deviation over seeds.}
\label{fig:gsm8k-sampled}
\end{figure}

\begin{table}[h]
\centering
\caption{GSM8K Pass@1 accuracy under varying entropy penalties, under greedy decoding and under temperature sampling ($T = 0.8$, $k = 8$ samples per problem). Model: \texttt{Qwen2.5-1.5B-Instruct}. The sampled column is evaluated on a coarser $\lambda$ grid and over 3 seeds; dashes mark configurations not run under sampling.}
\label{tab:gsm8k-greedy-vs-sampled}
\begin{tabular}{lcc}
\toprule
 & Greedy Pass@1 & Sampled Pass@1 ($T=0.8$) \\
\midrule
0.0 (Baseline) & 0.4912 $\pm$ 0.0154 & 0.3825 $\pm$ 0.0023 \\
0.3 & 0.5056 $\pm$ 0.0116 & -- \\
0.5 & 0.5080 $\pm$ 0.0064 & -- \\
1 & 0.5276 $\pm$ 0.0064 & 0.4942 $\pm$ 0.0144 \\
2 & 0.5388 $\pm$ 0.0086 & -- \\
3 & 0.5504 $\pm$ 0.0083 & 0.5246 $\pm$ 0.0060 \\
4 & 0.5536 $\pm$ 0.0134 & -- \\
6 & 0.5748 $\pm$ 0.0134 & -- \\
8 & \textbf{0.5816} $\pm$ 0.0161 & \textbf{0.5725} $\pm$ 0.0068 \\
\bottomrule
\end{tabular}
\end{table}

\begin{figure}[h]
\centering
\includegraphics[width=0.62\textwidth]{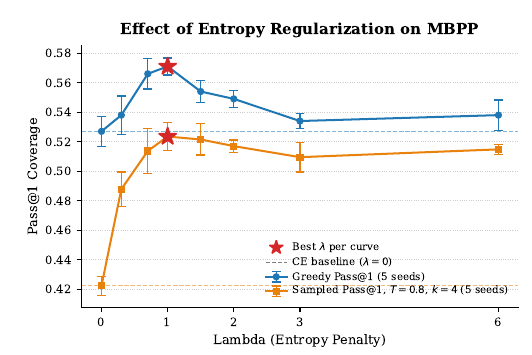}
\caption{MBPP Pass@1 under greedy decoding and under temperature sampling ($T = 0.8$, $k = 4$) versus entropy penalty $\lambda$. Dashed lines mark each curve's cross-entropy baseline; stars mark each curve's best $\lambda$. Error bars show $\pm 1$ standard deviation over 5 seeds.}
\label{fig:mbpp-sampled}
\end{figure}

\begin{table}[h]
\centering
\caption{MBPP Pass@1 accuracy under varying entropy penalties, under greedy decoding and under temperature sampling ($T = 0.8$, $k = 4$ samples per problem). Model: \texttt{Qwen2.5-1.5B-Instruct}.}
\label{tab:mbpp-greedy-vs-sampled}
\begin{tabular}{lcc}
\toprule
 & Greedy Pass@1 & Sampled Pass@1 ($T=0.8$) \\
\midrule
0.0 (Baseline) & 0.5270 $\pm$ 0.0103 & 0.4223 $\pm$ 0.0064 \\
0.3 & 0.5380 $\pm$ 0.0129 & 0.4878 $\pm$ 0.0117 \\
0.7 & 0.5660 $\pm$ 0.0102 & 0.5137 $\pm$ 0.0153 \\
1 & \textbf{0.5710} $\pm$ 0.0058 & \textbf{0.5235} $\pm$ 0.0093 \\
1.5 & 0.5540 $\pm$ 0.0073 & 0.5215 $\pm$ 0.0106 \\
2 & 0.5490 $\pm$ 0.0058 & 0.5170 $\pm$ 0.0041 \\
3 & 0.5340 $\pm$ 0.0049 & 0.5095 $\pm$ 0.0100 \\
6 & 0.5380 $\pm$ 0.0103 & 0.5148 $\pm$ 0.0033 \\
\bottomrule
\end{tabular}
\end{table}

\begin{table}[h]
\centering
\caption{Detailed per-seed sampled Pass@1 accuracy on MBPP at $T = 0.8$ with $k = 4$ samples per problem, using \texttt{Qwen2.5-1.5B-Instruct}.}
\label{tab:mbpp-sampled-seeds}
\begin{tabular}{lcccccrr}
\toprule
 & \multicolumn{5}{c}{Individual Seeds} & \multicolumn{2}{c}{Aggregate} \\
\cmidrule(lr){2-6} \cmidrule(lr){7-8}
Entropy Penalty ($\lambda$) & Seed 0 & Seed 1 & Seed 2 & Seed 3 & Seed 4 & Mean & Std. Dev. \\
\midrule
0.0 (Baseline) & 0.4113 & 0.4300 & 0.4238 & 0.4200 & 0.4263 & 0.4223 & $\pm$0.0064 \\
0.3 & 0.5038 & 0.4688 & 0.4950 & 0.4863 & 0.4850 & 0.4878 & $\pm$0.0117 \\
0.7 & 0.5413 & 0.5175 & 0.5100 & 0.4975 & 0.5025 & 0.5138 & $\pm$0.0153 \\
1 & 0.5312 & 0.5250 & 0.5100 & 0.5350 & 0.5162 & 0.5235 & $\pm$0.0093 \\
1.5 & 0.5125 & 0.5175 & 0.5100 & 0.5300 & 0.5375 & 0.5215 & $\pm$0.0106 \\
2 & 0.5212 & 0.5138 & 0.5225 & 0.5150 & 0.5125 & 0.5170 & $\pm$0.0041 \\
3 & 0.5125 & 0.5050 & 0.5162 & 0.4925 & 0.5212 & 0.5095 & $\pm$0.0100 \\
6 & 0.5175 & 0.5138 & 0.5162 & 0.5175 & 0.5088 & 0.5148 & $\pm$0.0033 \\
\bottomrule
\end{tabular}
\end{table}

\begin{figure}[h]
\centering
\includegraphics[width=0.62\textwidth]{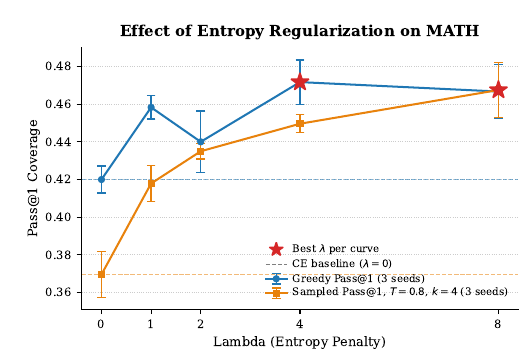}
\caption{MATH Pass@1 under greedy decoding and under temperature sampling ($T = 0.8$, $k = 4$) versus entropy penalty $\lambda$ for \texttt{Qwen2.5-7B-Instruct}. Dashed lines mark each curve's cross-entropy baseline; stars mark each curve's best $\lambda$. Error bars show $\pm 1$ standard deviation over 3 seeds.}
\label{fig:math-sampled}
\end{figure}

\begin{table}[h]
\centering
\caption{MATH Pass@1 accuracy under varying entropy penalties, under greedy decoding and under temperature sampling ($T = 0.8$, $k = 4$ samples per problem). Model: \texttt{Qwen2.5-7B-Instruct}.}
\label{tab:math-greedy-vs-sampled}
\begin{tabular}{lcc}
\toprule
 & Greedy Pass@1 & Sampled Pass@1 ($T=0.8$) \\
\midrule
0.0 (Baseline) & 0.4200 $\pm$ 0.0071 & 0.3696 $\pm$ 0.0121 \\
1 & 0.4583 $\pm$ 0.0062 & 0.4179 $\pm$ 0.0097 \\
2 & 0.4400 $\pm$ 0.0163 & 0.4350 $\pm$ 0.0041 \\
4 & \textbf{0.4717} $\pm$ 0.0118 & 0.4496 $\pm$ 0.0048 \\
8 & 0.4667 $\pm$ 0.0143 & \textbf{0.4675} $\pm$ 0.0145 \\
\bottomrule
\end{tabular}
\end{table}

\begin{table}[h]
\centering
\caption{Detailed per-seed sampled Pass@1 accuracy on MATH at $T = 0.8$ with $k = 4$ samples per problem, using \texttt{Qwen2.5-7B-Instruct}.}
\label{tab:math-sampled-seeds}
\begin{tabular}{lcccrr}
\toprule
 & \multicolumn{3}{c}{Individual Seeds} & \multicolumn{2}{c}{Aggregate} \\
\cmidrule(lr){2-4} \cmidrule(lr){5-6}
Entropy Penalty ($\lambda$) & Seed 0 & Seed 1 & Seed 2 & Mean & Std. Dev. \\
\midrule
0.0 (Baseline) & 0.3525 & 0.3775 & 0.3787 & 0.3696 & $\pm$0.0121 \\
1 & 0.4175 & 0.4300 & 0.4062 & 0.4179 & $\pm$0.0097 \\
2 & 0.4300 & 0.4400 & 0.4350 & 0.4350 & $\pm$0.0041 \\
4 & 0.4450 & 0.4475 & 0.4562 & 0.4496 & $\pm$0.0048 \\
8 & 0.4475 & 0.4813 & 0.4738 & 0.4675 & $\pm$0.0145 \\
\bottomrule
\end{tabular}
\end{table}

\paragraph{The gains are roughly twice as large under sampling.}
On GSM8K the entropy penalty buys $+9.04\%$ absolute under greedy decoding and $+19.00\%$ under sampling ($0.3825 \to 0.5725$). On MBPP it buys $+4.40\%$ greedy and $+10.12\%$ sampled ($0.4223 \to 0.5235$). On MATH it buys $+5.17\%$ greedy and $+9.79\%$ sampled ($0.3696 \to 0.4675$). The pattern is consistent: the metric that can see off-argmax probability mass records about twice the improvement of the metric that cannot. This is the signature the mechanism predicts.

\subsection{Why the Main Paper Reports Greedy Pass@1}

There are two reasons we use the greedy numbers in the main paper.

\paragraph{Greedy decoding performs better.} Greedy Pass@1 is higher than sampled Pass@1 at every $\lambda$ on every benchmark ($0.5816$ against $0.5725$ on GSM8K, $0.5710$ against $0.5235$ on MBPP, $0.4717$ against $0.4675$ on MATH). A practitioner deploying one of these models for a single attempt per problem would decode greedily.

\paragraph{Sampled evaluation at a fixed temperature confounds the policy with its effective temperature.} Minimizing entropy sharpens the predictive distribution. Sampling a sharpened distribution at $T = 0.8$ is, in its effect, similar to sampling the original distribution at some lower temperature. A sampled comparison at fixed $T$ therefore cannot on its own separate "the regularized policy places less mass on incorrect outputs" from "the regularized policy is being decoded at a lower effective temperature". Greedy decoding is immune to this by construction: the argmax is invariant to any monotone temperature rescaling, so greedy Pass@1 cannot be inflated by an effective-temperature effect.

\clearpage

\section{Renyi Order Ablation}
\label{app:renyi}

Section~4 derives the entropy penalty from the family $S_\alpha$ and then selects $\alpha = 1$, and Theorem~B.4 proves PAC-learnability for the $S_\alpha$ loss along a sequence $\alpha_m \downarrow 0$, with Corollary~B.5 instantiating $\alpha_m = 1/m$. The theory therefore points toward the small-$\alpha$ end of the family, where $\log S_\alpha$ approaches $\log|\mathrm{supp}(\cdot)|$, while the objective we actually train with sits at $\alpha = 1$. This appendix asks what the order parameter buys empirically.

\subsection{Objective and Protocol}

We train with
\[
\mathcal{L} \;=\; \mathrm{CE} \;+\; \lambda \, R_\alpha, \qquad
R_\alpha(p) \;=\; \frac{1}{1 - \alpha} \log \sum_{y} p(y)^\alpha ,
\]
evaluated in log space as $\mathrm{logsumexp}(\alpha \log p) / (1 - \alpha)$ for numerical stability, with $R_1$ taken to be the Shannon entropy, recovering the main objective exactly. We sweep $\alpha \in \{0.1, 0.25, 0.5, 1, 2\}$ and $\lambda \in \{0, 1, 3, 8\}$. At $\lambda = 0$ the penalty is switched off and the objective is identical for every $\alpha$, so that point is trained once and shared across all curves.

Everything else follows Appendix~C: GSM8K, \texttt{Qwen2.5-1.5B-Instruct}, the same LoRA configuration and hyperparameters, greedy Pass@1. Two differences from Table~1 are worth stating, because they shift the absolute numbers: this sweep uses 200 evaluation problems rather than 500, and 3 seeds rather than 5. The cross-entropy baseline accordingly reads $0.4717 \pm 0.0062$ here against $0.4912 \pm 0.0154$ in Table~1. Comparisons within this appendix are internally consistent; they should not be read against the main-paper tables.

\subsection{Results}

\begin{figure}[h]
\centering
\includegraphics[width=0.66\textwidth]{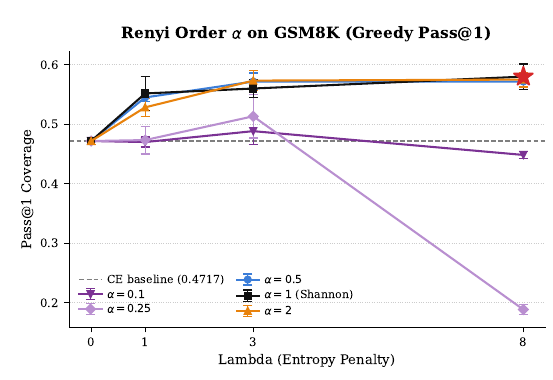}
\caption{Greedy Pass@1 on GSM8K for five Renyi orders across penalty strengths. The dashed line is the cross-entropy baseline, which all curves share at $\lambda = 0$. Error bars show $\pm 1$ standard deviation over 3 seeds. Orders $\alpha \in \{0.5, 1, 2\}$ are mutually indistinguishable; $\alpha = 0.1$ never separates from the baseline and $\alpha = 0.25$ collapses at $\lambda = 8$.}
\label{fig:renyi}
\end{figure}

\begin{table}[h]
\centering
\caption{Greedy Pass@1 on GSM8K by Renyi order $\alpha$ and penalty strength $\lambda$, aggregated over 3 seeds. The $\lambda = 0$ baseline ($0.4717 \pm 0.0062$) is shared by all orders.}
\label{tab:renyi-main}
\begin{tabular}{lccc}
\toprule
& $\lambda = 1$ & $\lambda = 3$ & $\lambda = 8$ \\
\midrule
$\alpha = 0.1$        & $0.4700 \pm 0.0082$ & $0.4883 \pm 0.0225$ & $0.4483 \pm 0.0062$ \\
$\alpha = 0.25$       & $0.4733 \pm 0.0232$ & $0.5133 \pm 0.0366$ & $0.1883 \pm 0.0085$ \\
$\alpha = 0.5$        & $0.5450 \pm 0.0071$ & $0.5717 \pm 0.0143$ & $0.5717 \pm 0.0103$ \\
$\alpha = 1$ (Shannon) & $0.5517 \pm 0.0287$ & $0.5600 \pm 0.0147$ & $\mathbf{0.5800 \pm 0.0212}$ \\
$\alpha = 2$          & $0.5283 \pm 0.0155$ & $0.5733 \pm 0.0170$ & $0.5750 \pm 0.0122$ \\
\bottomrule
\end{tabular}
\end{table}

\paragraph{Orders in $[0.5, 2]$ are one curve.}
Taking the largest minus the smallest mean across $\alpha \in \{0.5, 1, 2\}$ at each $\lambda$ and comparing it to the pooled standard error of that difference gives $1.2\sigma$ at $\lambda = 1$, $1.0\sigma$ at $\lambda = 3$ and $0.6\sigma$ at $\lambda = 8$. The ordering is not even stable: $\alpha = 1$ is highest at $\lambda = 1$ and $\lambda = 8$ but lowest at $\lambda = 3$. Within this range the Renyi order is not a meaningful design choice at three seeds.

\paragraph{Orders below $0.5$ fail.}
At $\lambda = 1$, $\alpha = 0.1$ reaches $0.4700$ and $\alpha = 0.25$ reaches $0.4733$, both indistinguishable from the un-regularized baseline of $0.4717$. At $\lambda = 8$, $\alpha = 0.25$ collapses to $0.1883$, roughly 28 points \emph{below} baseline, consistently across all three seeds.

\paragraph{The failure is not a numerical artifact.}
For $\alpha < 1$ the sum $\sum_y p(y)^\alpha$ up-weights the tail of the distribution, and with $|\mathcal{Y}| = 151{,}936$ one might worry that at small $\alpha$ the penalty reads the softmax's numerical floor rather than the model's support. We measured this directly, recording the share of $\sum_y p(y)^\alpha$ contributed by tokens outside the top 100. At $\alpha = 0.25$ that share is $0.053$, $0.032$ and $0.016$ at $\lambda = 1, 3, 8$; at $\alpha = 0.1$ it is $0.480$, $0.414$ and $0.359$. The order that fails most dramatically, $\alpha = 0.25$, is the one whose penalty is almost entirely determined by the head of the distribution. Whatever is going wrong at small $\alpha$ is a property of the objective, not of floating-point arithmetic.

\subsection{Reconciliation With Theorem B.4}

These results do not contradict Theorem~B.4; they locate the regime in which its constants bind. The theorem requires $\alpha_m \downarrow 0$ and $\lambda_m \uparrow \infty$ \emph{jointly with the sample size}, and its guarantee is mediated by Lemma~B.3, which bounds the gap between $|\mathrm{supp}(p)|$ and $S_\alpha(p)$ by $\left(1 - \eta^{\alpha / (1-\alpha)}\right) / \eta$, with a sample complexity scaling as $1/\eta^2$. In the LLM setting, the softmax assigns no exact zeros, so the effective $\eta$ is the smallest probability the model represents and $1/\eta$ is astronomically large. The sample sizes at which the small-$\alpha$ guarantees become operative are therefore far beyond anything reachable here, and at a fixed finite $\alpha$ and finite $m$ the theorem makes no prediction about which order should perform better.

The practical reading is that $\alpha = 1$ is not a compromise forced on us by tractability. It sits inside the range that works, and the end of the family the theory idealizes is empirically the worse end. Shannon entropy is both the convenient choice and, over the orders we tested, an effective one.

\subsection{Limitations}

Three caveats. This is one model on one dataset at three seeds, so the conclusion that $\alpha \in [0.5, 2]$ is a plateau should be read as an absence of evidence for an ordering rather than evidence of exact equivalence. The $\alpha = 0.1$ runs exhibited pre-clipping gradient norms of 7 to 10 against a clipping threshold of 1.0, roughly an order of magnitude above the other configurations, so the cross-entropy component of the gradient was attenuated and those cells were effectively optimized under a different regime; part of their deficit may be attributable to that rather than to the order itself. Finally, we have not diagnosed the mechanism of the $\alpha = 0.25$, $\lambda = 8$ collapse, and we report it as an observed instability rather than as a characterized failure mode.

\clearpage

\section{Answer-Only Negative Control}
\label{app:answeronly}

Our account of why the entropy penalty helps is specific: it withdraws probability mass from outputs the verifier rejects, which matters only when the demonstrated answer is one of many acceptable ones. If the correct set is a singleton, cross-entropy already concentrates all mass on the single acceptable output and a support penalty has nothing left to correct. The theory therefore makes a falsifiable prediction - no real gain when the demonstrated answer is the only accepted one - and this appendix tests it.

We test it by removing the chain of thought from GSM8K. The model is trained to emit the final number and nothing else, so the target is the unique correct answer rather than one derivation among many. Everything else is held fixed: the same dataset, the same \texttt{Qwen2.5-1.5B-Instruct} model, the same LoRA configuration and hyperparameters, the same $\lambda$ grid, the same five seeds, the same trainer, the same greedy Pass@1 verifier, the same 1{,}000 training and 500 evaluation problems. The only difference from Section~5.1 is the format of the target.

\begin{figure}[h]
\centering
\includegraphics[width=0.95\textwidth]{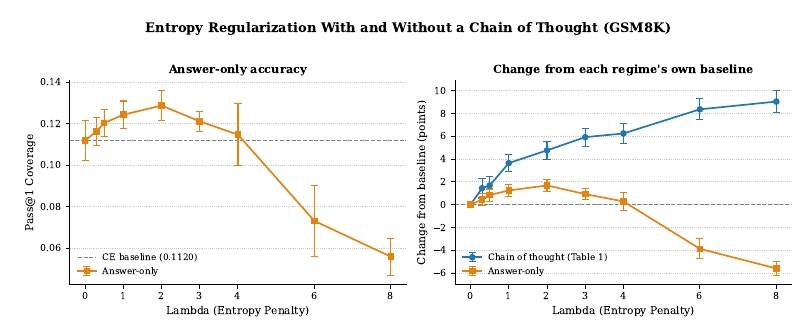}
\caption{GSM8K with and without a chain of thought. Left: answer-only greedy Pass@1 against the entropy penalty, with its cross-entropy baseline dashed. Right: change from each regime's own baseline, so the two can be compared despite very different absolute accuracies. The chain-of-thought curve rises monotonically to $+9.04$ points; the answer-only curve is flat through $\lambda = 4$ and then falls to $-5.60$. Error bars show $\pm 1$ standard deviation over 5 seeds, and on the right the standard error of the difference from baseline.}
\label{fig:answeronly}
\end{figure}
\begin{table}[h]
\centering
\caption{GSM8K answer-only greedy Pass@1 under varying entropy penalties, aggregated over 5 seeds, with the change from the cross-entropy baseline in percentage points (95\% Welch interval).}
\label{tab:answeronly}
\begin{tabular}{lcc}
\toprule
$\lambda$ & Pass@1 & Change (95\% CI) \\
\midrule
0.0 (Baseline) & $0.1120 \pm 0.0097$ & --- \\
0.3 & $0.1164 \pm 0.0067$ & $+0.44\ [-0.72, +1.60]$ \\
0.5 & $0.1204 \pm 0.0067$ & $+0.84\ [-0.32, +2.00]$ \\
1.0 & $0.1244 \pm 0.0066$ & $+1.24\ [+0.09, +2.39]$ \\
2.0 & $0.1288 \pm 0.0074$ & $+1.68\ [+0.48, +2.88]$ \\
3.0 & $0.1212 \pm 0.0047$ & $+0.92\ [-0.14, +1.98]$ \\
4.0 & $0.1148 \pm 0.0148$ & $+0.28\ [-1.45, +2.01]$ \\
6.0 & $0.0732 \pm 0.0170$ & $-3.88\ [-5.80, -1.96]$ \\
8.0 & $0.0560 \pm 0.0089$ & $\mathbf{-5.60}\ [-6.89, -4.31]$ \\
\bottomrule
\end{tabular}
\end{table}

\paragraph{Results.}
The answer-only baseline reaches $0.1120$. Accuracy drifts up slightly through $\lambda = 2$, by $+1.68$ points with a 95\% interval of $[+0.48, +2.88]$, which does not survive correction for the eight comparisons in the sweep and which we do not claim as an effect. It then falls sharply: $-3.88$ points $[-5.80, -1.96]$ at $\lambda = 6$ and $-5.60$ points $[-6.89, -4.31]$ at $\lambda = 8$, halving accuracy relative to the cross-entropy baseline. This stark contrast relative to the clear advantage entropy penalties gave the reasoning model provides evidence that the penalty works due to the mechanism in Section~3.

\clearpage

\section{LLM Disclosure}
LLMs were used in 3 primary ways. The first is essentially all experiments were set up and instantiated with LLM coding. We had a separate LLM check for correctness of the code. The second is when the paper was drafted, we repeatedly asked LLMs to critique and we incorporated its feedback if we deemed it to be appropriate. The third was finding related papers. We read all of these before ever citing them to confirm the relation to our work.
\end{document}